\documentclass[twoside,11pt]{article}
\usepackage{jmlr2e}
\usepackage{amsmath,amsfonts,amssymb}
\usepackage[ruled]{algorithm2e}
\usepackage{units}
\usepackage{psfrag}

\newcommand{\reals}{\mathbb{R}}
\newcommand{\E}{\mathbb{E}}
\newcommand{\Var}{\mathrm{Var}}
\newcommand{\norm}[1]{\|#1\|}
\newcommand{\inner}[1]{\langle#1\rangle}

\newcommand{\ceil}[1]{\left\lceil#1\right\rceil}
\newcommand{\secref}[1]{Section~\ref{#1}}
\newcommand{\figref}[1]{Figure~\ref{#1}}
\renewcommand{\eqref}[1]{Equation~(\ref{#1})}
\newcommand{\lemref}[1]{Lemma~\ref{#1}}
\newcommand{\thmref}[1]{Theorem~\ref{#1}}
\newcommand{\appref}[1]{Appendix~\ref{#1}}
\newcommand{\algref}[1]{Algorithm~\ref{#1}}
\newcommand{\Gcal}{\mathcal{G}}
\newcommand{\Tcal}{\mathcal{T}}
\newcommand{\Vcal}{\mathcal{V}}
\newcommand{\Zcal}{\mathcal{Z}}
\newcommand{\minimize}{\mathop{\mathrm{minimize}{}}}

\newcommand{\argmin}{\mathop{\mathrm{arg\,min}{}}}

\newcommand{\dom}{\mathrm{dom\,}}

\jmlrheading{13}{2012}{165-202}{11/10, Revised 10/11}{1/12}
{Ofer Dekel, Ran Gilad-Bachrach, Ohad Shamir and Lin Xiao}
\firstpageno{165}
\ShortHeadings{Optimal Distributed Online Prediction}{Dekel, Gilad-Bachrach, Shamir and Xiao}

\title{Optimal Distributed Online Prediction Using Mini-Batches}

\author{\name Ofer Dekel \email oferd@microsoft.com \\
\name Ran Gilad-Bachrach \email rang@microsoft.com \\
\addr Microsoft Research \\
      1 Microsoft Way \\
      Redmond, WA 98052, USA
\AND
\name Ohad Shamir \email ohadsh@microsoft.com \\
\addr Microsoft Research \\
      1 Memorial Drive \\
      Cambridge, MA 02142, USA
\AND
\name Lin Xiao \email lin.xiao@microsoft.com\\
\addr Microsoft Research \\
      1 Microsoft Way \\
      Redmond, WA 98052, USA
}

\editor{Tong Zhang}

\begin{document}
\maketitle

\begin{abstract}%
Online prediction methods are typically presented as serial algorithms
running on a single processor. However, in the age of web-scale
prediction problems, it is increasingly common to encounter situations
where a single processor cannot keep up with the high rate at which
inputs arrive. In this work, we present the \emph{distributed
  mini-batch} algorithm, a method of converting many serial
gradient-based online prediction algorithms into distributed
algorithms.  We prove a regret bound for this method that is
asymptotically optimal for smooth convex loss functions and stochastic
inputs. Moreover, our analysis explicitly takes into account
communication latencies between nodes in the distributed environment.
We show how our method can be used to solve the closely-related
distributed stochastic optimization problem, achieving an
asymptotically linear speed-up over multiple processors. Finally, we
demonstrate the merits of our approach on a web-scale online
prediction problem.
\end{abstract}


\section{Introduction}

Many natural prediction problems can be cast as stochastic online
prediction problems. These are often discussed in the serial setting,
where the computation takes place on a single processor. However, when
the inputs arrive at a high rate and have to be processed in real
time, there may be no choice but to distribute the computation across
multiple cores or multiple cluster nodes. For example, modern search
engines process thousands of queries a second, and indeed they are
implemented as distributed algorithms that run in massive
data-centers. In this paper, we focus on such \emph{large-scale} and
\emph{high-rate} online prediction problems, where parallel and
distributed computing is critical to providing a real-time service.

First, we begin by defining the stochastic online prediction problem.
Suppose that we observe a stream of inputs $z_{1},z_{2},\ldots$, where
each~$z_{i}$ is sampled independently from a fixed unknown
distribution over a sample space $\Zcal$. Before observing each
$z_{i}$, we predict a point $w_{i}$ from a set~$W$.  After
making the prediction~$w_{i}$, we observe~$z_{i}$ and suffer the loss
$f(w_{i},z_{i})$, where $f$ is a predefined loss function.  Then we
use~$z_{i}$ to improve our prediction mechanism for the future (e.g.,
using a stochastic gradient method). The goal is to accumulate the
smallest possible loss as we process the sequence of inputs.  More
specifically, we measure the quality of our predictions using the
notion of \emph{regret}, defined as
\[
R(m) ~=~ \sum_{i=1}^m \left( f(w_i,z_i) - f(w^\star,z_i) \right) ~,
\]
where $w^\star = \argmin_{w \in W} \E_z[f(w,z)]$.  Regret
measures the difference between the cumulative loss of our predictions
and the cumulative loss of the fixed predictor $w^\star$, which is
optimal with respect to the underlying distribution.  Since regret
relies on the stochastic inputs $z_i$, it is a random variable.  For
simplicity, we focus on bounding the expected regret $\E[R(m)]$, and
later use these results to obtain high-probability bounds on the
actual regret. In this paper, we restrict our discussion to convex
prediction problems, where the loss function $f(w,z)$ is convex in~$w$ for
every $z\in\Zcal$, and $W$ is a closed convex subset of $\reals^n$.

Before continuing, we note that the stochastic online
\emph{prediction} problem is closely related, but not identical, to
the stochastic \emph{optimization} problem
\citep[see, e.g.,][]{Wets89,BirgeLouveaux97,NemirovskiJuLaSh09}.
The main difference between the two is in their goals:
in stochastic optimization, the goal is to generate a sequence
$w_1,w_2,\ldots$ that quickly converges to the minimizer of the function
$F(\cdot) = \E_z[f(\cdot,z)]$.
The motivating application is usually a static (batch) problem, and not an
online process that occurs over time. Large-scale static optimization problems
can always be solved using a serial approach, at the cost of a longer
running time. In online prediction, the goal is to generate a sequence
of predictions that accumulates a small loss along the way, as
measured by regret. The relevant motivating application here is
providing a real-time service to users, so our algorithm must keep up
with the inputs as they arrive, and we cannot choose to slow down. In
this sense, distributed computing is critical for large-scale online
prediction problems.  Despite these important differences, our
techniques and results can be readily adapted to the stochastic online
optimization setting.

We model our distributed computing system as a set of~$k$
\emph{nodes}, each of which is an independent processor, and a
\emph{network} that enables the nodes to communicate with each other.
Each node receives an incoming stream of examples from an outside
source, such as a load balancer/splitter. As in the real world, we
assume that the network has a limited bandwidth, so the nodes cannot
simply share all of their information, and that messages sent over the
network incur a non-negligible latency. However, we assume that
network operations are \emph{non-blocking}, meaning that each node can
continue processing incoming traffic while network operations complete
in the background.

How well can we perform in such a distributed environment?  At one
extreme, an ideal (but unrealistic) solution to our problem is to run
a serial algorithm on a single ``super'' processor that is~$k$ times
faster than a standard node. This solution is optimal, simply because
any distributed algorithm can be simulated on a fast-enough single
processor.  It is well-known that the optimal regret bound that can be
achieved by a gradient-based serial algorithm on an arbitrary convex
loss is $O(\sqrt{m})$
\citep[e.g.,][]{NemirovskiYu83,CesaBianchiLu06,AbernethyAgRaBa09}.  At
the other extreme, a trivial solution to our problem is to have each
node operate in isolation of the other $k\!-\!1$ nodes, running an
independent copy of a serial algorithm, without any communication over
the network.  We call this the \emph{no-communication} solution.  The
main disadvantage of this solution is that the performance guarantee,
as measured by regret, scales poorly with the network size~$k$.  More
specifically, assuming that each node processes $m/k$ inputs, the
expected regret per node is $O(\sqrt{m/k})$. Therefore, the total
regret across all $k$ nodes is $O(\sqrt{km})$ - namely, a factor of
$\sqrt{k}$ worse than the ideal solution. The first sanity-check that
any distributed online prediction algorithm must pass is that it
outperforms the na\"ive no-communication solution.

In this paper, we present the \emph{distributed mini-batch} (DMB) algorithm,
a method of converting any serial gradient-based online prediction algorithm
into a parallel or distributed algorithm.
This method has two important properties:
\begin{itemize}
\item It can use any gradient-based update rule for serial online
  prediction as a black box, and convert it into a parallel or
  distributed online prediction algorithm.
\item If the loss function~$f(w,z)$ is smooth in~$w$ (see the precise
  definition in \eqref{eqn:smoothness}), then our method attains an
  asymptotically optimal regret bound of $O(\sqrt{m})$.  Moreover, the
  coefficient of the dominant term~$\sqrt{m}$ is the same as in the
  serial bound, and \emph{independent} of~$k$ and of the network
  topology.
\end{itemize}
The idea of using mini-batches in stochastic and online learning is
not new, and has been previously explored in both the serial and
parallel settings
\citep[see, e.g.,][]{SSS07,GimpelDasSmith10}.
However, to the
best of our knowledge, our work is the first to use this idea to
obtain such strong results in a parallel and distributed learning
setting (see \secref{sec:relatedwork} for a comparison to related
work).

Our results build on the fact that the optimal regret bound for serial
stochastic gradient-based prediction algorithms can be refined if the
loss function is smooth. In particular, it can be shown that the
hidden coefficient in the $O(\sqrt{m})$ notation is proportional to
the standard deviation of the stochastic gradients evaluated at each
predictor~$w_i$ \citep{JuditskyNT11,Lan09,Xiao10}.
We make the key observation
that this coefficient can be effectively reduced by averaging a
mini-batch of stochastic gradients computed at the same predictor, and
this can be done in parallel with simple network communication.
However, the non-negligible communication latencies prevent a
straightforward parallel implementation from obtaining the optimal
serial regret bound.\footnote{For example, if the network
  communication operates over a minimum-depth spanning tree and the
  diameter of the network scales as $\log(k)$, then we can show that a
  straightforward implementation of the idea of parallel variance
  reduction leads to an $O\bigl(\sqrt{m \log(k)}\bigr)$ regret
  bound. See~\secref{sec:distlearnsync} for details.}  In order to
close the gap, we show that by letting the mini-batch size grow slowly
with~$m$, we can attain the optimal $O(\sqrt{m})$ regret bound, where
the dominant term of order~$\sqrt{m}$ is \emph{independent} of the
number of nodes~$k$ and of the latencies introduced by the network.

The paper is organized as follows. In \secref{sec:preliminaries}, we
present a template for stochastic gradient-based serial prediction
algorithms, and state refined variance-based regret bounds for smooth
loss functions.  In \secref{sec:serial}, we analyze the effect of
using mini-batches in the serial setting, and show that it does not
significantly affect the regret bounds.  In
\secref{sec:distlearnsync}, we present the DMB algorithm, and show
that it achieves an asymptotically optimal serial regret bound for
smooth loss functions.  In \secref{sec:optimization}, we show that the
DMB algorithm attains the optimal rate of convergence for stochastic
optimization, with an asymptotically linear speed-up.  In
\secref{sec:experiments}, we complement our theoretical results with
an experimental study on a realistic web-scale online prediction
problem.  While substantiating the effectiveness of our approach, our
empirical results also demonstrate some interesting properties of
mini-batching that are not reflected in our theory.  We conclude with
a comparison of our methods to previous work in
\secref{sec:relatedwork}, and a discussion of potential extensions and
future research in \secref{sec:conclusions}.
The main topics presented in this paper are summarized in
\citet{DGSX11}. \citet{DGSX11} also present robust variants of our approach, which are resilient to failures and node heterogeneity in an asynchronous distributed environment.

\section{Variance Bounds for Serial Algorithms}
\label{sec:preliminaries}

\begin{algorithm}[t]
\For{$j = 1,2,\ldots$}
{
  predict $w_j$\;
  receive input $z_j$ sampled i.i.d. from unknown distribution\;
  suffer loss $f(w_j,z_j)$\;
  define $g_j = \nabla_w f(w_j,z_j)$\;
  compute $(w_{j+1},a_{j+1})=\phi\left(a_{j},g_{j},\alpha_j\right)$\;
}
\caption{Template for a serial first-order stochastic online prediction
  algorithm.}
\label{alg:serial}
\end{algorithm}

Before discussing distributed algorithms, we must fully understand the
serial algorithms on which they are based.  We focus on gradient-based
optimization algorithms that follow the template outlined in
\algref{alg:serial}.  In this template, each prediction is made by an
unspecified \emph{update rule}:
\begin{equation}\label{eqn:updaterule}
(w_{j+1}, a_{j+1}) = \phi( a_j, g_j, \alpha_j ).
\end{equation}
The update rule~$\phi$ takes three arguments: an auxiliary state
vector~$a_j$ that summarizes all of the necessary information about
the past, a gradient~$g_j$ of the loss function $f(\cdot,z_j)$
evaluated at~$w_j$, and an iteration-dependent parameter~$\alpha_j$
such as a stepsize.  The update rule outputs the next predictor
$w_{j+1}\in W$ and a new auxiliary state vector~$a_{j+1}$.
Plugging in different update rules results in different online
prediction algorithms. For simplicity, we assume for now that the
update rules are deterministic functions of their inputs.

As concrete examples, we present two well-known update rules that fit
the above template. The first is the \emph{projected gradient descent}
update rule,
\begin{equation}\label{eqn:gradec}
w_{j+1} = \pi_{W}\left(w_{j}-\frac{1}{\alpha_{j}} g_{j} \right),
\end{equation}
where $\pi_{W}$ denotes the Euclidean projection onto the set
$W$.  Here $1/\alpha_{j}$ is a decaying learning rate,
with~$\alpha_j$ typically set to be $\Theta(\sqrt{j})$.  This fits the
template in \algref{alg:serial} by defining $a_j$ to simply be $w_j$,
and defining $\phi$ to correspond to the update rule specified
in \eqref{eqn:gradec}.  We note that the projected gradient method is
a special case of the more general class of \emph{mirror descent}
algorithms \citep[e.g.,][]{NemirovskiJuLaSh09,Lan09}, which all fit in
the template of~\eqref{alg:serial}.

Another family of update rules that fit in our setting is the \emph{dual
  averaging} method \citep{Nesterov09,Xiao10}.  A dual averaging update rule
takes the form
\begin{equation}\label{eqn:DA}
w_{j+1} = \argmin_{w \in W} \left\{ \left\langle
\sum_{i=1}^j g_i, w \right\rangle + \alpha_j\, h(w) \right\} ~,
\end{equation}
where $\langle\cdot,\cdot\rangle$ denotes the vector inner product,
$h:W \to \reals$ is a strongly convex auxiliary function,
and $\alpha_j$ is a monotonically increasing sequence of positive
numbers, usually set to be $\Theta(\sqrt{j})$.  The dual averaging
update rule fits the template in \algref{alg:serial} by defining
$a_j$ to be $\sum_{i=1}^jg_i$.  In the special case where
$h(w)=(1/2)\|w\|_2^2$, the minimization problem in~\eqref{eqn:DA} has
the closed-form solution
\begin{equation}\label{eqn:EuclideanDA}
w_{j+1} = \pi_{W}\left( -\frac{1}{\alpha_j}\sum_{i=1}^j g_j \right).
\end{equation}

For stochastic online prediction problems with convex loss functions,
both of these update rules have expected regret bound of
$O(\sqrt{m})$.  In general, the coefficient of the dominant $\sqrt{m}$
term is proportional to an upper bound on the expected norm of the
stochastic gradient \citep[e.g.,][]{Zinkevich03}.
Next we present refined bounds for smooth convex loss functions,
which enable us to develop optimal distributed algorithms.

\subsection{Optimal Regret Bounds for Smooth Loss Functions}
\label{sec:serial-bound}

As stated in the introduction, we assume that the loss function
$f(w,z)$ is convex in~$w$ for each~$z\in\Zcal$ and
that~$W$ is a closed convex set.
We use $\|\cdot\|$ to denote the Euclidean norm in~$\reals^n$.
For convenience, we use the notation $F(w)=\E_z [f(w,z)]$ and assume
$w^\star=\argmin_{w\in W} F(w)$ always exists.  Our main
results require a couple of additional assumptions:

\begin{itemize}
\item \emph{Smoothness} - we assume that $f$ is $L$-smooth in
  its first argument, which means that for any $z \in \Zcal$, the
  function $f(\cdot,z)$ has $L$-Lipschitz continuous
  gradients. Formally,
\begin{equation}
\forall\, z \in \Zcal, \quad \forall\,w,w' \in W, \qquad
\norm{\nabla_w f(w,z) - \nabla_w f(w',z)} \leq L \norm{w - w'}~.
\label{eqn:smoothness}
\end{equation}
\item \emph{Bounded Gradient Variance} - we assume that $\nabla_w
  f(w,z)$ has a $\sigma^2$-bounded variance for any fixed $w$,
  when~$z$ is sampled from the underlying distribution. In other words,
  we assume that there exists a constant $\sigma \geq 0$ such that
\[
\forall\, w \in W, \qquad \E_z \left[ \bigl\| \nabla_w f(w,z)
- \nabla F(w) ] \bigr\|^2 \right] \leq \sigma^2 ~.
\]
\end{itemize}

Using these assumptions, regret bounds that explicitly depend on the
gradient variance can be established
\citep{JuditskyNT11,Lan09,Xiao10}.
In particular, for the projected stochastic gradient method
defined in~\eqref{eqn:gradec}, we have the following result:

\begin{theorem}\label{thm:gd-bound}
Let $f(w,z)$ be an $L$-smooth convex loss function in~$w$ for each $z\in\Zcal$
and assume that the stochastic gradient $\nabla_w f(w,z)$ has
$\sigma^2$-bounded variance for all $w\in W$.
In addition, assume that~$W$ is convex and bounded,
and let $D = \sqrt{\max_{u,v\in W} \|u-v\|^2/2}$.  Then
using $\alpha_j=L+(\sigma/D)\sqrt{j}$ in \eqref{eqn:gradec} gives
\[
\E [R(m)] ~\leq~ \bigl( F(w_1)-F(w^\star)\bigr) + D^2 L + 2 D \sigma \sqrt{m}.
\]
\end{theorem}

In the above theorem, the assumption that~$W$ is a bounded set does
not play a critical role. Even if the learning problem has no
constraints on~$w$, we could always confine the search to a bounded
set (say, a Euclidean ball of some radius) and \thmref{thm:gd-bound}
guarantees an $O(\sqrt{m})$ regret compared to the optimum within that
set.

Similarly, for the dual averaging method defined in \eqref{eqn:DA}, we have:

\begin{theorem}\label{thm:da-bound}
Let $f(w,z)$ be an $L$-smooth convex loss function in~$w$ for each
$z\in\Zcal$, assume that the stochastic gradient $\nabla_w f(w,z)$ has
$\sigma^2$-bounded variance for all $w\in W$, and let
$D=$ \linebreak[4] $\sqrt{h(w^\star) - \min_{w\in W}h(w)}$. Then, by setting
$w_1=\argmin_{w\in W} h(w)$ and
$\alpha_j=L+(\sigma/D)\sqrt{j}$ in the dual averaging method we have
\[
\E [R(m)] ~\leq~ \bigl( F(w_1)-F(w^\star)\bigr) + D^2 L + 2 D \sigma \sqrt{m}.
\]\end{theorem}

For both of the above theorems, if $\nabla F(w^\star)=0$ (which is certainly
the case if~$W=\reals^n$), then the expected regret bounds can be
simplified to
\begin{equation}\label{eqn:serialBound}
\E [R(m)] ~\leq~ 2 D^2 L  + 2 D \sigma \sqrt{m} ~.
\end{equation}
Proofs for these two theorems, as well as the above simplification,
are given in~\appref{app:serial}.  Although we focus on expected
regret bounds here, our results can equally be stated as
high-probability bounds on the actual regret
(see~\appref{app:highprob} for details).

In both \thmref{thm:gd-bound} and \thmref{thm:da-bound}, the
parameters $\alpha_j$ are functions of~$\sigma$.  It may be difficult
to obtain precise estimates of the gradient variance in many concrete
applications. However, note that any upper bound on the variance
suffices for the theoretical results to hold, and identifying such a
bound is often easier than precisely estimating the actual variance. A
loose bound on the variance will increase the constants in our regret
bounds, but will not change its qualitative $O(\sqrt{m})$ rate.

Euclidean gradient descent and dual averaging are not the only update
rules that can be plugged into \algref{alg:serial}.  The analysis
in~\appref{app:serial} (and~\appref{app:highprob}) actually applies to
a much larger class of update rules, which includes the family of
mirror descent updates \citep{NemirovskiJuLaSh09, Lan09} and the
family of (non-Euclidean) dual averaging updates \citep{Nesterov09,
  Xiao10}.  For each of these update rules, we get an expected regret
bound that closely resembles the bound in \eqref{eqn:serialBound}.

Similar results can also be established for loss functions of the form
$f(w,z)+\Psi(w)$, where $\Psi(w)$ is a simple convex regularization term that
is not necessarily smooth.
For example, setting $\Psi(w)=\lambda\|w\|_1$ with $\lambda>0$ promotes
sparsity in the predictor~$w$.
To extend the dual averaging method, we can use the following
update rule in \citet{Xiao10}:
\[
w_{j+1} = \argmin_{w\in W} \left\{ \left\langle \frac{1}{j}
\sum_{i=1}^j g_i, \,w\right\rangle + \Psi(w) + \frac{\alpha_j}{j} h(w)\right\}.
\]
Similar extensions to the mirror descent method can be found in, for example,
\citet{DuchiSinger09}.
Using these composite forms of the algorithms, the same regret bounds
as in \thmref{thm:gd-bound} and \thmref{thm:da-bound} can be achieved
even if~$\Psi(w)$ is nonsmooth.
The analysis is almost identical to \appref{app:serial} by using the
general framework of \citet{Tseng08}.

Asymptotically, the bounds we presented in this section are only controlled
by the variance $\sigma^2$ and the number of iterations $m$.
Therefore, we can think of any of the bounds mentioned above as an abstract
function $\psi(\sigma^2, m)$, which we assume to be monotonically
increasing in its arguments.

\subsection{Analyzing the No-Communication Parallel Solution}
\label{sec:idealtrivial}

Using the abstract notation $\psi(\sigma^2,m)$ for the expected
regret bound simplifies our presentation significantly.  As an
example, we can easily give an analysis of the no-communication
parallel solution described in the introduction.

In the na\"ive no-communication solution, each of the~$k$ nodes in the
parallel system applies the same serial update rule to its own substream
of the high-rate inputs, and no communication takes place between them.
If the total number of examples processed by the~$k$ nodes is~$m$, then
each node processes at most $\ceil{m/k}$ inputs. The examples received by
each node are i.i.d.\ from the original distribution, with the same
variance bound~$\sigma^2$ for the stochastic gradients.
Therefore, each node suffers an expected regret of at most
$\psi(\sigma^2, \ceil{m/k})$ on its portion of the input
stream, and the total regret bound is obtain by simply summing over
the~$k$ nodes, that is,
\[
\E[R(m)] ~\leq~
k\,\psi \left( \sigma^2, \ceil{\frac{m}{k}} \right).
  \]
If $\psi(\sigma^2, m)=2D^2L+2D\sigma\sqrt{m}$,
as in~\eqref{eqn:serialBound}, then the expected total regret is
\[
\E[R(m)] ~\leq~
2k D^2 L +2 D \sigma k\sqrt{\ceil{\frac{m}{k}}}.
\]
Comparing this bound to $2D^2L+2D\sigma\sqrt{m}$ in the ideal serial
solution, we see that it is approximately $\sqrt{k}$ times worse in its
leading term.
This is the price one pays for the lack of communication in the distributed
system.
In \secref{sec:distlearnsync}, we show how this $\sqrt{k}$ factor can
be avoided by our DMB approach.

\section{Serial Online Prediction using Mini-Batches}
\label{sec:serial}

The expected regret bounds presented in the previous section depend
on the variance of the stochastic gradients.
The explicit dependency on the variance naturally suggests the idea of
using averaged gradients over mini-batches to reduce the variance.
Before we present the distributed mini-batch algorithm in the next
section, we first analyze a \emph{serial} mini-batch algorithm.

\begin{algorithm}[t]
\For{$j = 1,2,\ldots$}
{
  initialize $\bar g_j := 0$\;
  \For{$s = 1,\ldots,b$}
  {
    define $i := (j-1)b+s$\;
    predict $w_j$\;
    receive input $z_i$ sampled i.i.d. from unknown distribution\;
    suffer loss $f(w_j, z_i)$\;
    $g_i := \nabla_w f(w_j, z_{i})$\;
    $\bar g_j := \bar g_j +  (1/b) g_i$\;
  }
  set $(w_{j+1},a_{j+1}) = \phi \big(a_j, \bar{g_j},\alpha_j\big)$\;
}
\caption{Template for a serial mini-batch algorithm.}
\label{alg:minibatch}
\end{algorithm}

In the setting described in \algref{alg:serial}, the update rule is
applied after each input is received. We deviate from this setting and
apply the update only periodically. Letting $b$ be a user-defined
\emph{batch size} (a positive integer), and considering every~$b$
consecutive inputs as a \emph{batch}. We define the \emph{serial mini-batch
  algorithm} as follows: Our prediction remains constant for the duration of each batch, and is updated only when a batch ends.  While processing the
$b$ inputs in batch $j$, the algorithm calculates and accumulates
gradients and defines the average gradient
\[
\bar g_j = \frac{1}{b} \sum_{s=1}^b \nabla_w f(w_j, z_{(j-1)b+s}) ~.
\]
Hence, each batch of $b$ inputs generates a single average gradient.
Once a batch ends, the serial mini-batch algorithm feeds $\bar g_j$ to
the update rule $\phi$ as the $j^\text{th}$ gradient and
obtains the new prediction for the next batch and the new state. See
\algref{alg:minibatch} for a formal definition of the serial
mini-batch algorithm.
The appeal of the serial mini-batch setting is that the update rule is used
less frequently, which may have computational benefits.

\begin{theorem} \label{thm:minibatch}
Let $f(w,z)$ be an $L$-smooth convex loss function in~$w$ for each
$z\in\Zcal$ and assume that
the stochastic gradient $\nabla_w f(w,z_i)$ has $\sigma^2$-bounded
variance for all $w$.  If the update rule $\phi$ has the serial
regret bound $\psi(\sigma^2, m)$, then the expected regret
of~\algref{alg:minibatch} over~$m$ inputs is at most
\[
b\,\psi\left(\frac{\sigma^2}{b},\left\lceil\frac{m}{b}\right\rceil\right).
\]
If $\psi(\sigma^2, m)=2D^2L+2D\sigma\sqrt{m}$,
then the expected regret is bounded by
\[
2bD^2L+2D\sigma\sqrt{m+b}.
\]
\end{theorem}

\begin{proof}
Assume without loss of generality that $b$ divides $m$, and that the
serial mini-batch algorithm processes exactly $m/b$ complete
batches.\footnote{ We can make this assumption since if $b$ does not
  divide $m$ then we can pad the input sequence with additional inputs
  until $m/b = \ceil{m/b}$, and the expected regret can only
  increase.}  Let~$\Zcal^b$ denote the set of all sequences of $b$
elements from $\Zcal$, and assume that a sequence is sampled from
$\Zcal^b$ by sampling each element i.i.d.\ from $\Zcal$. Let $\bar f:
W \times \Zcal^b \mapsto \reals$ be defined as
\[
\bar f\left(w, \left(z_1,\ldots,z_b\right)\right)
~=~ \frac{1}{b} \sum_{s=1}^b f(w, z_s) ~.
\]
In other words, $\bar f$ averages the loss function~$f$ across $b$
inputs from~$\Zcal$, while keeping the prediction constant.  It is
straightforward to show that
$\E_{\bar{z}\in\Zcal^b} \bar{f}(w,\bar{z})
= \E_{z\in\Zcal} f(w,z) = F(w)$.

Using the linearity of the gradient operator, we have
$$
\nabla_w \bar f\left(w, \left(z_1,\ldots,z_b\right)\right)
= \frac{1}{b} \sum_{s=1}^b \nabla_w f\left(w, z_s\right)~.
$$
Let $\bar z_j$ denote the sequence $(z_{(j-1)b+1},\ldots,z_{jb})$,
namely, the sequence of~$b$ inputs in batch~$j$.  The vector~$\bar
g_j$ in~\algref{alg:minibatch} is precisely the gradient of $\bar
f(\cdot,\bar z_j)$ evaluated at~$w_j$.  Therefore the serial
mini-batch algorithm is equivalent to using the update
rule~$\phi$ with the loss function~$\bar f$.

Next we check the properties of $\bar{f}(w,\bar{z})$ against the two
assumptions in \secref{sec:serial-bound}.
First, if $f$ is $L$-smooth then $\bar f$ is $L$-smooth as well
due to the triangle inequality.
Then we analyze the variance of the stochastic gradient.
Using the properties of the Euclidean norm, we can write
\begin{eqnarray*}
\left\| \nabla_w \bar{f}(w,\bar{z}) - \nabla F(w) \right\|^2
&=& \bigg\| \frac{1}{b}\sum_{s=1}^b \left( \nabla_w f(w,z_s) - \nabla F(w)
    \right) \bigg\|^2 \\
&=& \frac{1}{b^2} \sum_{s=1}^b \sum_{s'=1}^b \Big\langle
 \nabla_w f(w,z_s) - \nabla F(w), \nabla_w f(w,z_{s'})-\nabla F(w) \Big\rangle.
\end{eqnarray*}
Notice that $z_s$ and $z_{s'}$ are independent whenever $s\neq s'$, and in
such cases,
\begin{eqnarray*}
&& \E \Big\langle \nabla_w f(w,z_s)-\nabla F(w),
    \nabla_w f(w,z_{s'})-\nabla F(w) \Big\rangle \\
&=&\Big\langle \E \big[\nabla_w f(w,z_s)-\nabla F(w)\big],
    ~\E\big[\nabla_w f(w,z_{s'})-\nabla F(w)\big] \Big\rangle
~~=~~ 0.
\end{eqnarray*}
Therefore, we have for every $w\in W$,
\begin{equation}\label{eqn:var-scaling}
\E \left\| \nabla_w \bar{f}(w,\bar{z}) - \nabla F(w) \right\|^2
~=~ \frac{1}{b^2}\sum_{s=1}^b \E \big\| \left( \nabla_w f(w,z_s) - \nabla F(w)
    \right) \big\|^2 \\
~\leq~ \frac{\sigma^2}{b}.
\end{equation}
So we conclude that
$\nabla_w \bar f(w,\bar z_j)$ has a $(\sigma^2/b)$-bounded
variance for each~$j$ and each $w\in W$.
If the update rule~$\phi$ has a regret
bound $\psi(\sigma^2,m)$ for the loss function~$f$ over~$m$
inputs, then its regret for~$\bar f$ over
$m/b$ batches is bounded as
\[
\E \bigg[ \sum_{j=1}^{m/b} \bigl( \bar f(w_j,\bar z_j)
- \bar f(w^\star,\bar z_j) \bigr) \bigg]
~\leq~ \psi\left( \frac{\sigma^2}{b}, \frac{m}{b} \right).
\]
By replacing~$\bar f$ above with its definition, and multiplying both
sides of the above inequality by~$b$, we have
\[
\E \bigg[ \sum_{j=1}^{m/b} \sum_{i=(j-1)b+1}^{jb}
\bigl( f(w_j,z_i) - f(w^\star,z_i) \bigr) \bigg]
~\leq~ b\, \psi\left(
\frac{\sigma^2}{b}, \frac{m}{b} \right).
\]

If $\psi(\sigma^2, m)=2D^2L+2D\sigma\sqrt{m}$, then
simply plugging in the general bound
$b\,\psi(\nicefrac{\sigma^2}{b},\lceil\nicefrac{m}{b}\rceil)$
and using $\lceil \nicefrac{m}{b} \rceil \leq \nicefrac{m}{b}+1$ gives
the desired result. However, we note that the optimal algorithmic
parameters, as specified in~\thmref{thm:gd-bound}
and~\thmref{thm:da-bound}, must be changed to
$\alpha_j=L+(\nicefrac{\sigma}{\sqrt{b}D})\sqrt{j}$ to reflect the
reduced variance $\sigma^2/b$ in the mini-batch setting.
\end{proof}

The bound in \thmref{thm:minibatch} is asymptotically equivalent to
the $2D^2L+2D\sigma\sqrt{m}$ regret bound for the basic serial
algorithms presented in~\secref{sec:preliminaries}. In other words,
performing the mini-batch update in the serial setting does not
significantly hurt the performance of the update rule.  On the other
hand, it is also not surprising that using mini-batches in the serial
setting does not improve the regret bound. After all, it is still a
serial algorithm, and the bounds we presented
in~\secref{sec:serial-bound} are optimal.
Nevertheless, our experiments demonstrate that in real-world scenarios,
mini-batching can in fact have a very substantial positive
effect on the transient performance of the online prediction algorithm, even in
the serial setting (see \secref{sec:experiments} for details).
Such positive effects are not captured by our asymptotic, worst-case analysis.

\section{Distributed Mini-Batch for Stochastic Online Prediction}
\label{sec:distlearnsync}

In this section, we show that in a distributed setting, the mini-batch
idea can be exploited to obtain nearly optimal regret bounds.
To make our setting as realistic as possible, we assume that any
communication over the network incurs a latency.
More specifically, we view the network as an undirected graph~$\Gcal$ over
the set of nodes, where each edge represents a bi-directional network link.
If nodes~$u$ and~$v$ are not connected by a link, then any communication
between them must be relayed through other nodes.
The latency incurred between~$u$ and~$v$ is therefore proportional to the
graph distance between them, and the longest possible latency is thus
proportional to the diameter of~$\Gcal$.

In addition to latency, we assume that the network has limited
bandwidth. However, we would like to avoid the tedious discussion of
data representation, compression schemes, error correcting, packet
sizes, etc. Therefore, we do not explicitly quantify the bandwidth of
the network. Instead, we require that the communication load at each
node remains constant, and does not grow with the number of nodes~$k$
or with the rate at which the incoming functions arrive.

Although we are free to use any communication model that respects the
constraints of our network, we assume only the availability of  a distributed
vector-sum operation.  This is a
standard\footnote{For example, all-reduce with the sum operation is a
  standard operation in MPI.} synchronized network operation. Each
vector-sum operation begins with each node holding a vector $v_j$, and
ends with each node holding the sum $\sum_{j=1}^k v_j$. This operation
transmits messages along a rooted minimum-depth spanning-tree of
$\Gcal$, which we denote by $\Tcal$: first the leaves of $\Tcal$ send
their vectors to their parents; each parent sums the vectors received
from his children and adds his own vector; the parent then sends the
result to his own parent, and so forth; ultimately the sum of all
vectors reaches the tree root; finally, the root broadcasts the
overall sum down the tree to all of the nodes.

An elegant property of the vector-sum operation is that it uses each up-link
and each down-link in $\Tcal$ exactly once. This allows us to start
vector-sum operations back-to-back. These vector-sum operations will run
concurrently without creating network congestion on any edge of~$\Tcal$.
Furthermore, we assume that the network operations are \emph{non-blocking},
meaning that each node can continue processing incoming inputs while the
vector-sum operation takes place in the background.
This is a key property that allows us to efficiently deal with network latency.
To formalize how latency affects the performance of our algorithm,
let~$\mu$ denote the number of inputs that are processed by the entire system
during the period of time it takes to complete a vector-sum operation
across the entire network.
Usually~$\mu$ scales linearly with the diameter of the network,
or (for appropriate network architectures) logarithmically in the number
of nodes~$k$.

\subsection{The DMB Algorithm}
\label{sec:dmb-algorithm}

\begin{algorithm}[p]
 \For{$j = 1,2,\ldots$}
 {
   initialize $\hat g_j := 0$\;
   \For{$s = 1,\ldots,b/k$}
   {
     predict $w_j$\;
     receive input $z$ sampled i.i.d. from unknown distribution\;
     suffer loss $f(w_j,z)$\;
     compute $g := \nabla_w f(w_j,z)$\;
     $\hat g_j := \hat g_j + g$\;
   }
  call the distributed vector-sum to compute the sum of $\hat g_j$ across all nodes\;
   receive $\mu/k$ additional inputs and continue predicting using $w_j$\;
   finish vector-sum and compute average gradient $\bar g_j$ by dividing the sum by~$b$\;
   set $(w_{j+1},a_{j+1}) = \phi \big(a_{j}, \bar{g}_{j},\alpha_j\big)$\;
 }
\caption{Distributed mini-batch (DMB) algorithm (running on each node).}
\label{alg:distributed}
\end{algorithm}

\begin{figure}[p]
\centering
\psfrag{1}[bc]{$1$}
\psfrag{1}[bc]{$1$}
\psfrag{2}[bc]{$2$}
\psfrag{ldots}[bc]{\ldots}
\psfrag{k}[bc]{$k$}
\psfrag{wj}[cr]{$w_j$}
\psfrag{wj1}[cr]{$w_{j+1}$}
\psfrag{b}[cl]{$b$}
\psfrag{mu}[cl]{$\mu$}
\includegraphics[width=0.4\textwidth]{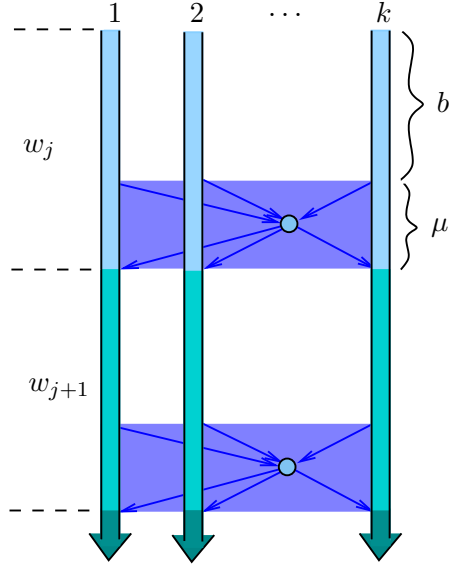}
\caption{Work flow of the DMB algorithm. Within each
  batch~$j=1,2,\ldots$, each node accumulates the stochastic gradients
  of the first $b/k$ inputs.  Then a vector-sum operation across the
  network is used to compute the average across all nodes. While the
  vector-sum operation completes in the background, a total of~$\mu$
  inputs are processed by the processors using the same
  predictor~$w_j$, but their gradients are not collected. Once all of
  the nodes have the overall average $\bar{g}_j$, each node updates
  the predictor using the same deterministic serial algorithm.}
\label{fig:dmb}
\end{figure}

We are now ready to present a general technique for applying a
deterministic update rule~$\phi$ in a distributed environment.
This technique resembles the serial mini-batch technique described
earlier, and is therefore called the \emph{distributed mini-batch}
algorithm, or DMB for short.

Algorithm~\ref{alg:distributed} describes a template of the DMB algorithm
that runs in parallel on each node in the network, and Figure~\ref{fig:dmb}
illustrates the overall algorithm work-flow.
Again, let~$b$ be a batch size, which we will specify later on, and for
simplicity assume that~$k$ divides~$b$ and~$\mu$. The DMB algorithm
processes the input stream in batches $j=1,2,\ldots$, where each batch
contains $b + \mu$ consecutive inputs. During each batch $j$, all of
the nodes use a common predictor $w_j$. While observing the
first~$b$ inputs in a batch, the nodes calculate and accumulate the
stochastic gradients of the loss function~$f$ at~$w_j$. Once the nodes
have accumulated~$b$ gradients altogether, they start a distributed
vector-sum operation to calculate the sum of these $b$
gradients. While the vector-sum operation completes in the background,
$\mu$ additional inputs arrive (roughly $\mu/k$ per node) and the
system keeps processing them using the same predictor~$w_j$.  The
gradients of these additional $\mu$ inputs are discarded (to this end,
they do not need to be computed). Although this may seem wasteful, we
show that this waste can be made negligible by choosing~$b$
appropriately.

Once the vector-sum operation completes, each node holds the sum of
the~$b$ gradients collected during batch~$j$.  Each node divides this
sum by~$b$ and obtains the average gradient, which we denote by~$\bar
g_j$. Each node feeds this average gradient to the update rule
$\phi$, which returns a new synchronized prediction~$w_{j+1}$.
In summary, during batch~$j$ each node processes $(b+\mu)/k$
inputs using the same predictor~$w_j$, but only the first $b/k$
gradients are used to compute the next predictor. Nevertheless, all
$b+\mu$ inputs are counted in our regret calculation.

If the network operations are conducted over a spanning tree, then an obvious
variants of the DMB algorithm is to let the root apply the update rule to get
the next predictor, and then broadcast it to all other nodes.
This saves repeated executions of the update rule at each node
(but requires interruption or modification of the standard
vector-sum operations in the network communication model).
Moreover, this guarantees all the nodes having the same predictor even
with update rules that depends on some random bits.

\begin{theorem} \label{thm:synchronous}
Let $f(w,z)$ be an $L$-smooth convex loss function in~$w$ for each
$z\in\Zcal$ and assume that
the stochastic gradient $\nabla_w f(w,z_i)$ has $\sigma^2$-bounded variance
for all~$w\in W$.
If the update rule~$\phi$ has the serial regret bound
$\psi(\sigma^2, m)$, then the expected regret of
\algref{alg:distributed} over~$m$ samples is at most
\[
(b+\mu)\,\psi\left(\frac{\sigma^2}{b},
\left\lceil\frac{m}{b+\mu}\right\rceil\right) ~.
\]
Specifically, if $\psi(\sigma^2, m)=2D^2L+2D\sigma\sqrt{m}$,
then setting the batch size $b = m^{\nicefrac{1}{3}}$ gives the expected
regret bound
\begin{equation}\label{eq:optimalbound}
2D\sigma\sqrt{m} + 2Dm^{\nicefrac{1}{3}}\left(L D+\sigma\sqrt{\mu}\right)
+ 2D\sigma m^{\nicefrac{1}{6}} +2D\sigma \mu m^{-\nicefrac{1}{6}} + 2\mu D^2 L.
\end{equation}
In fact, if~$b=m^\rho$ for any $\rho\in(0,1/2)$, the expected regret bound is
$2D\sigma\sqrt{m} + o(\sqrt{m})$.
\end{theorem}

To appreciate the power of this result, we compare the specific
bound in \eqref{eq:optimalbound} with the ideal serial solution and the
na\"ive no-communication solution discussed in the introduction.  It
is clear that our bound is asymptotically equivalent to the ideal
serial bound $\psi(\sigma^2,m)$---even the constants in the
dominant terms are identical.  Our bound scales nicely with the
network latency and the cluster size $k$, because $\mu$ (which usually scales
logarithmically with~$k$) does not appear
in the dominant $\sqrt{m}$ term.  On the other hand, the na\"ive
no-communication solution has regret bounded by
$k\psi\left(\sigma^2, \nicefrac{m}{k}\right) =
2kD^2L+2D\sigma\sqrt{km}$ (see \secref{sec:idealtrivial}).  If
$1\ll k \ll m$, this bound is worse than the bound in
\thmref{thm:synchronous} by a factor of $\sqrt{k}$.

Finally, we note that choosing~$b$ as $m^\rho$ for an
appropriate~$\rho$ requires knowledge of~$m$ in advance. However, this
requirement can be relaxed by applying a standard doubling trick
\citep{CesaBianchiLu06}.  This gives a single algorithm that does not
take~$m$ as input, with asymptotically similar regret.  If we use a
fixed~$b$ regardless of~$m$, the dominant term of the regret bound
becomes $2D\sigma\sqrt{\log(k) m/b}$; see the following proof for
details.

\vspace{1ex}
\begin{proof}
Similar to the proof of \thmref{thm:minibatch}, we assume without loss of generality that
$k$ divides $b + \mu$, we define
the function $\bar f: W \times \Zcal^b \mapsto \reals$ as
\[
\bar f\left(w, \left(z_1,\ldots,z_b\right)\right)
~=~ \frac{1}{b} \sum_{s=1}^b f(w, z_s) ~~,
\]
and we use $\bar z_j$ to denote the \emph{first~$b$ inputs} in batch~$j$.
By construction, the function~$\bar f$ is $L$-smooth and its gradients
have $\nicefrac{\sigma^2}{b}$-bounded variance.  The average
gradient~$\bar g_j$ computed by the DMB algorithm is the gradient of
$\bar f(\cdot,\bar z_j)$ evaluated at the point $w_j$.  Therefore,
\begin{equation}\label{eqn:avgbound}
\E \bigg[ \sum_{j=1}^{m/(b+\mu)} \bigl( \bar f(w_j,\bar z_j)
- \bar f(w^\star,\bar z_j) \bigr) \bigg]
~\leq~ \psi\left( \frac{\sigma^2}{b}, \frac{m}{b+\mu} \right).
\end{equation}
This inequality only involve the additional~$\mu$ examples in counting
the number of batches as $\nicefrac{m}{b+\mu}$.
In order to count them in the total regret, we notice that
\[
\forall\,j, \quad \E \left[ \bar f(w_j,\bar z_j) \,|\, w_j \right]
= \E \bigg[ \frac{1}{b+\mu} \sum_{i=(j-1)(b+\mu)+1}^{j(b+\mu)}
f(w_j, z_i) ~\bigg|~ w_j \bigg] ,
\]
and a similar equality holds for $\bar f(w^\star, z_i)$.
Substituting these equalities in the left-hand-side of \eqref{eqn:avgbound} and multiplying both sides by~$b+\mu$ yields
\[
\E \bigg[ \sum_{j=1}^{ m/(b+\mu) }
\sum_{i=(j-1)(b+\mu)+1}^{j(b+\mu)}
\bigl( f(w_j,z_i) - f(w^\star,z_i) \bigr) \bigg]
~\leq~ (b+\mu) \, \psi \! \left(
\frac{\sigma^2}{b},\frac{m}{b+\mu}\right).
\]
Again, if $(b+\mu)$ divides~$m$, then the left-hand side above is
exactly the expected regret of the DMB algorithm over~$m$ examples.
Otherwise, the expected regret can only be smaller.

For the concrete case of
$\psi(\sigma^2, m)=2D^2L+2D\sigma\sqrt{m}$,
plugging in the new values for $\sigma^2$ and $m$ results in a bound of the form
\begin{eqnarray*}
(b+\mu) \psi\left(\frac{\sigma^2}{b}, \left\lceil\frac{m}{b+\mu}\right\rceil
\right)
&\leq& (b+\mu) \psi\left(\frac{\sigma^2}{b}, \frac{m}{b+\mu} + 1 \right)  \\
&\leq& 2(b+\mu)D^2L +2D\sigma\sqrt{m+\frac{\mu}{b}m+\frac{(b+\mu)^2}{b}}.
\end{eqnarray*}
Using the inequality $\sqrt{x+y+z}\leq \sqrt{x}+\sqrt{y}+\sqrt{z}$, which holds for any
nonnegative numbers $x$, $y$ and $z$, we bound the expression above by
\[
2(b+\mu) D^2 L + 2D\sigma\sqrt{m}
+ 2D\sigma\sqrt{\frac{\mu m}{b}}+2D\sigma\frac{b+\mu}{\sqrt{b}}.
\]
It is clear that with $b=C m^{\rho}$ for any $\rho\in (0,1/2)$ and any
constant~$C>0$, this bound can be written as $2D\sigma\sqrt{m}+o(\sqrt{m})$.
Letting $b=m^{1/3}$ gives the smallest exponents in the $o(\sqrt{m})$ terms.
\end{proof}

In the proofs of \thmref{thm:minibatch} and
\thmref{thm:synchronous}, decreasing the variance by a factor
of~$b$, as given in \eqref{eqn:var-scaling}, relies on
properties of the Euclidean norm.  For serial gradient-type algorithms
that are specified with different norms (see the general framework in
\appref{app:serial}), the variance does not typically decrease as
much. For example, in the dual averaging method specified in \eqref{eqn:DA},
if we use $h(w)=1/(2(p-1))\|w\|_p^2$ for some $p\in(1,2]$,
then the ``variance'' bounds for the stochastic gradients must be
expressed in the dual norm, that is,
$\E \left\| \nabla_w f(w,z) - \nabla F(w) \right\|_q^2 \leq \sigma^2$,
where $q=p/(p-1)\in[2,\infty)$.
In this case, the variance bound for the averaged function becomes
\[
\E \left\| \nabla_w \bar{f}(w,\bar{z}) - \nabla F(w) \right\|_q^2
~\leq~ C(n,q) \frac{\sigma^2}{b},
\]
where $C(n,q)=\min\{q-1,O(\log(n))\}$ is a space-dependent constant.\footnote{
For further details of algorithms using $p$-norm,
see \citet[Section~7.2]{Xiao10} and \citet{ShalevShwartzTewari11}.
For the derivation of $C(n,q)$ see for instance Lemma B.2 in \cite{minibatch}.}
Nevertheless, we can still obtain a linear reduction in~$b$ even for
such non-Euclidean norms.
The net effect is that the regret bound for the DMB algorithm becomes
$2D\sqrt{C(n,q)}\sigma\sqrt{m} + o(\sqrt{m})$.

\subsection{Improving Performance on Short Input Streams}
\label{sec: short inputs}

Theorem \ref{thm:synchronous} presents an optimal way of choosing the batch
size $b$, which results in an asymptotically optimal regret
bound. However, our asymptotic approach hides a potential shortcoming
that occurs when $m$ is small. Say that we know, ahead of time, that
the sequence length is $m = 15,000$.  Moreover, say that the latency
is $\mu = 100$, and that $\sigma=1$ and $L=1$.  In this case,
\thmref{thm:synchronous} determines that the optimal batch size is $b
\sim 25$. In other words, for every $25$ inputs that participate in
the update, $100$ inputs are discarded. This waste becomes negligible
as~$b$ grows with~$m$ and does not affect our asymptotic analysis.
However, if~$m$ is known to be small, we can take steps to improve the situation.

Assume for simplicity that $b$ divides $\mu$. Now, instead of running
a single distributed mini-batch algorithm, we run $c= 1 + \mu/b$
independent interlaced instances of the distributed mini-batch
algorithm on each node. At any given moment, $c-1$ instances are
asleep and one instance is active. Once the active instance collects
$b/k$ gradients on each node, it starts a vector-sum network
operation, awakens the next instance, and puts itself to sleep. Note
that each instance awakens after $(c-1)b = \mu$ inputs, which is just
in time for its vector-sum operation to complete.

In the setting described above, $c$ different vector-sum operations
propagate concurrently through the network. The distributed vector sum operation is
typically designed such that each network link is used at most once in
each direction, so concurrent sum operations that begin at
different times should not compete for network resources.  The batch
size should indeed be set such that the generated traffic does not
exceed the network bandwidth limit, but the latency of each sum
operation should not be affected by the fact that multiple sum
operations take place at once.

Simply interlacing $c$ independent copies of our algorithm does not
resolve the aforementioned problem, since each prediction is still
defined by $1/c$ of the observed inputs. Therefore, instead of using
the predictions prescribed by the individual online predictors, we use
their average.  Namely, we take the most recent prediction generated
by each instance, average these predictions, and use this average in
place of the original prediction.

The advantages of this modification are not apparent from our
theoretical analysis. Each instance of the algorithm handles $m/c$ inputs and suffers a regret
of at most
$$
b \,\psi \left( \frac{\sigma^2}{b}, 1 + \frac{m}{bc} \right) ~~,
$$
and, using Jensen's inequality, the overall regret using the average prediction is upper bounded by
$$
bc \,\psi \left( \frac{\sigma^2}{b}, 1 + \frac{m}{bc} \right) ~~.
$$
The bound above is precisely the same as the bound in
\thmref{thm:synchronous}.  Despite this fact, we conjecture that this method will indeed improve empirical
results when the batch size $b$ is small compared to the latency term $\mu$.

\section{Stochastic Optimization}
\label{sec:optimization}

As we discussed in the introduction, the \emph{stochastic
  optimization} problem is closely related, but not identical, to the
stochastic online prediction problem.  In both cases, there is a loss
function $f(w,z)$ to be minimized.  The difference is in the way
success is measured.  In online prediction, success is measured by
regret, which is the difference between the cumulative loss suffered
by the prediction algorithm and the cumulative loss of the best fixed
predictor. The goal of stochastic optimization is to find an
approximate solution to the problem
\[
\minimize_{w\in W} \quad F(w) \triangleq \E_z [ f(w,z) ] ~,
\]
and success is measured by the difference between the expected loss of
the final output of the optimization algorithm and the expected loss
of the true minimizer~$w^\star$.  As before, we assume that the loss
function $f(w,z)$ is convex in~$w$ for any $z\in\Zcal$, and that~$W$
is a closed convex set.

We consider the same \emph{stochastic approximation} type of
algorithms presented in~\algref{alg:serial}, and define the final
output of the algorithm, after processing~$m$ i.i.d.\ samples, to be
\[
\bar w_m = \frac{1}{m} \sum_{j=1}^m w_j ~.
\]
In this case, the appropriate measure of success is the optimality gap
\[
G(m) ~=~ F(\bar w_m)-F(w^\star) ~.
\]
Notice that the optimality gap $G(m)$ is also a random variable,
because~$\bar w_m$ depends on the random samples $z_1,\ldots,z_m$.
It can be shown \citep[see, e.g.,][Theorem 3]{Xiao10} that for convex loss
functions and i.i.d.\ inputs, we always have
\[
\E[G(m)] ~\leq~ \frac{1}{m} \E[R(m)] ~.
\]
Therefore, a bound on the expected optimality gap can be readily
obtained from a bound on the expected regret of the same algorithm.
In particular, if~$f$ is an $L$-smooth convex loss function and
$\nabla_w f(w,z)$ has $\sigma^2$-bounded variance, and our
algorithm has a regret bound of $\psi(\sigma^2,m)$, then it also has an
expected optimality gap of at most
\[
\bar\psi(\sigma^2,m) = \frac{1}{m} \psi(\sigma^2, m) ~.
\]
For the specific regret bound $\psi(\sigma^2,m)=2D^2L + 2D\sigma\sqrt{m}$, which
holds for the serial algorithms presented in \secref{sec:preliminaries},
we have
\[
\E[G(m)] ~\leq~ \bar\psi(\sigma^2,m)
~=~ \frac{2D^2 L}{m} +\frac{2D\sigma}{\sqrt{m}}~.
\]

\subsection{Stochastic Optimization using Distributed Mini-Batches}

\begin{algorithm}[t]
$r \leftarrow \left\lfloor\frac{m}{b}\right\rfloor$\;

 \For{$j = 1,2,\ldots,r$}
 {
   reset $\hat g_j=0$\;
   \For{$s = 1,\ldots,b/k$}
   {
     receive input $z_s$ sampled i.i.d.\ from unknown distribution\;
     calculate $g_s = \nabla_w f(w_j,z_s)$\;
     calculate $\hat g_j \leftarrow \hat g_j + g_i$\;
   }
   start distributed vector sum to compute the sum of $\hat g_j$ across all nodes\;
   finish distributed vector sum and compute average gradient $\bar g_j$\;
   set $(w_{j+1},a_{j+1}) = \phi \big(a_{j}, \bar{g}_{j},j\big)$\;
 }
\KwOut{$\frac{1}{r}\sum_{j=1}^{r} w_{j}$}
 \caption{Template of DMB algorithm for stochastic optimization.}
 \label{alg:opt-DMB}
\end{algorithm}

Our template of a DMB algorithm for stochastic optimization (see
\algref{alg:opt-DMB}) is very similar to the one presented for the
online prediction setting.  The main difference is that we do not have
to process inputs while waiting for the vector-sum network operation
to complete.  Again let~$b$ be the batch size, and the number of
batches $r=\lfloor\nicefrac{m}{b}\rfloor$. For simplicity of
discussion, we assume that~$b$ divides~$m$.

\begin{theorem}\label{thm:opt-DMB}
Let $f(w,z)$ be an $L$-smooth convex loss function in~$w$ for each
$z\in\Zcal$ and assume that
the stochastic gradient $\nabla_{w}f(w,z)$ has $\sigma^{2}$-bounded
variance for all $w\in W$.  If the update rule~$\phi$
used in a serial setting has an expected optimality gap bounded by
$\bar\psi(\sigma^2,m)$, then the expected optimality gap of
\algref{alg:opt-DMB} after processing~$m$ samples is at most
\[
\bar\psi \left(\frac{\sigma^{2}}{b},\frac{m}{b} \right) ~.
\]
If
$\bar\psi(\sigma^{2},m)=\frac{2D^2 L}{m}+\frac{2D\sigma}{\sqrt{m}}$,
then the expected optimality gap is bounded by
\[
\frac{2 b D^2 L}{m}+\frac{2D\sigma}{\sqrt{m}} ~.
\]
\end{theorem}
The proof of the theorem follows along the lines of \thmref{thm:minibatch},
and is omitted.

We comment that the accelerated stochastic gradient methods of
\citet{Lan09}, \citet{HuKwokPan09} and \citet{Xiao10}
can also fit in our template for the DMB algorithm,
but with more sophisticated updating rules.
These accelerated methods have an expected optimality bound of
$\bar\psi(\sigma^{2},m)=\nicefrac{4D^2 L}{m^2}+\nicefrac{4D\sigma}{\sqrt{m}}$, which translates into the following bound for the DMB algorithm:
\[
\bar\psi\left(\frac{\sigma^{2}}{b},\frac{m}{b}\right)
=\frac{4b^2D^2 L}{m^2}+\frac{4D\sigma}{\sqrt{m}}~.
\]
Most recently, \citet{GhadimiLan10} developed accelerated stochastic
gradient methods for strongly convex functions that have the
convergence rate
$\bar \psi(\sigma^2, m) =
O(1) \left( \nicefrac{L}{m^2} + \nicefrac{\sigma^2}{\nu m} \right)$,
where $\nu$ is the strong convexity parameter of the loss function.
The corresponding DMB algorithm has a convergence rate
\[
\bar \psi \left( \frac{\sigma^2}{b}, \frac{m}{b} \right) =
O(1) \left( \frac{b^2 L}{m^2} + \frac{\sigma^2}{\nu m} \right).
\]
Apparently, this also fits in the DMB algorithm nicely.

The significance of our result is that the dominating factor in the
convergence rate is not affected by the batch size.  Therefore,
depending on the value of~$m$, we can use large batch sizes without
affecting the convergence rate in a significant way.  Since we can run
the workload associated with a single batch in parallel, this theorem
shows that the mini-batch technique is capable of turning many serial
optimization algorithms into parallel ones.  To this end, it is
important to analyze the speed-up of the parallel algorithms in terms
of the running time (wall-clock time).

\subsection{Parallel Speed-Up}
Recall that $k$ is the number of parallel computing nodes and $m$ is
the total number of i.i.d.\ samples to be processed.  Let~$b(m)$ be
the batch size that depends on~$m$.  We define a \emph{time-unit} to
be the time it takes a single node to process one sample (including
computing the gradient and updating the predictor). For convenience,
let~$\delta$ be the latency of the vector-sum operation in the network
(measured in number of time-units).\footnote{ The
  relationship between~$\delta$ and~$\mu$ defined in the online
  setting (see Section \ref{sec:distlearnsync}) is roughly
  $\mu=k\delta$.} Then the parallel speed-up of the DMB algorithm is
$$
S(m) = \frac{m}{\frac{m}{b(m)} \left(\frac{b(m)}{k} + \delta \right)}
=\frac{k}{1+ \frac{\delta}{b(m)} k} ~,
$$
where $m/b(m)$ is the number of batches, and $b(m)/k+\delta$ is the
wall-clock time by~$k$ processors to finish one batch in the DMB algorithm.
If $b(m)$ increases at a fast enough rate, then we have
$S(m)\to k$ as $m\to\infty$.
Therefore, we obtain an asymptotically linear speed-up, which is the ideal
result that one would hope for in parallelizing the optimization process
\citep[see][]{Gustafson88}.

In the context of stochastic optimization, it is more appropriate to
measure the speed-up with respect to the same optimality gap, not the
same amount of samples processed.  Let~$\epsilon$ be a given target
for the expected optimality gap.  Let $m_\mathrm{srl}(\epsilon)$ be
the number of samples that the serial algorithm needs to reach this
target and let $m_\mathrm{DMB}(\epsilon)$ be the number of samples
needed by the DMB algorithm.  Slightly overloading our notation, we
define the parallel speed-up with respect to the expected optimality
gap~$\epsilon$ as
\begin{equation}\label{eqn:speed-up}
S(\epsilon) = \frac{m_\mathrm{srl}(\epsilon)}
{\frac{m_\mathrm{DMB}(\epsilon)}{b}
\left(\frac{b}{k} + \delta \right)}~.
\end{equation}
In the above definition, we intentionally leave the dependence of~$b$ on~$m$
unspecified.
Indeed, once we fix the function~$b(m)$, we can substitute it into the equation
$\bar\psi(\nicefrac{\sigma^2}{b},\nicefrac{m}{b})=\epsilon$ to solve
for the exact form of $m_\mathrm{DMB}(\epsilon)$.
As a result, $b$ is also a function of~$\epsilon$.

Since both $m_\mathrm{srl}(\epsilon)$ and $m_\mathrm{DMB}(\epsilon)$ are upper
bounds for the actual running times to reach $\epsilon$-optimality, their ratio
$S(\epsilon)$ may not be a precise measure of the speed-up.
However, it is difficult in practice to measure the actual running times of
the algorithms in terms of reaching $\epsilon$-optimality.
So we only hope $S(\epsilon)$ gives a conceptual guide in comparing
the actual performance of the algorithms.
The following result shows that if the batch size~$b$ is chosen to be of
order $m^\rho$ for any $\rho\in(0,1/2)$, then we still have asymptotic
linear speed-up.

\begin{theorem}\label{thm:speedup}
Let $f(w,z)$ be an $L$-smooth convex loss function in~$w$ for each
$z\in\Zcal$ and assume that
the stochastic gradient $\nabla_{w}f(w,z)$ has $\sigma^{2}$-bounded
variance for all $w\in W$.
Suppose the update rule~$\phi$ used in the serial setting has an
expected optimality gap bounded by
$\bar\psi(\sigma^{2},m)=\frac{2D^2 L}{m}+\frac{2D\sigma}{\sqrt{m}}$.
If the batch size in the DMB algorithm is chosen as
$b(m)=\Theta(m^\rho)$, where $\rho\in(0,1/2)$, then we have
\[
\lim_{\epsilon\to 0} S(\epsilon) = k.
\]
\end{theorem}
\begin{proof}
By solving the equation
\[
\frac{2 D^2 L}{m}+\frac{2D\sigma}{\sqrt{m}}  = \epsilon ~,
\]
we see that the following number of samples is sufficient for the
serial algorithm to reach $\epsilon$-optimality:
\[
m_\mathrm{srl}(\epsilon) ~=~ \frac{D^2\sigma^2}{\epsilon^2}
\left(1+\sqrt{1+\frac{2L\epsilon}{\sigma^2}} \right)^2 ~.
\]
For the DMB algorithm, we use the batch size
$b(m)=(\nicefrac{\theta\sigma}{DL}) m^{\rho}$,
with some $\theta>0$, to obtain the equation
\begin{equation}\label{eqn:DMB-m}
\frac{2 b(m) D^2 L}{m}+\frac{2D\sigma}{\sqrt{m}}
~=~ \frac{2 D\sigma}{m^{1/2}} \left( 1 + \frac{\theta}{m^{1/2-\rho}}\right)
~=~ \epsilon.
\end{equation}
We use $m_\mathrm{DMB}(\epsilon)$ to denote the solution of the above equation.
Apparently $m_\mathrm{DMB}(\epsilon)$ is a monotone function of~$\epsilon$
and $\lim_{\epsilon\to 0} m_\mathrm{DMB}(\epsilon) = \infty$.
For convenience (with some abuse of notation),
let $b(\epsilon)$ to denote $b(m_\mathrm{DMB}(\epsilon))$,
which is also monotone in~$\epsilon$ and satisfies
$\lim_{\epsilon\to 0} b(\epsilon) = \infty$.
Moreover, for any batch size $b>1$, we have
$m_\mathrm{DMB}(\epsilon) \geq m_\mathrm{srl}(\epsilon)$.
Therefore, from Equation~(\ref{eqn:speed-up}) we get
\[
\limsup_{\epsilon\to 0} S(\epsilon)
\leq \lim_{\epsilon\to 0} \frac{k}{1+\frac{\delta}{b(\epsilon)} k}
= k.
\]

Next we show $\liminf_{\epsilon\to 0} S(\epsilon) \geq k$.
For any $\eta>0$, let
\[
m_\eta(\epsilon) = \frac{4D^2\sigma^2(1+\eta)^2}{\epsilon^2}.
\]
which is monotone decreasing in~$\epsilon$, and
can be seen as the solution to the equation
\[
\frac{2 D\sigma}{m^{1/2}} \left( 1 + \eta\right) ~=~ \epsilon.
\]
Comparing this equation with Equation~(\ref{eqn:DMB-m}), we see that,
for any $\eta>0$, there exists an $\epsilon'$ such that
for all $0<\epsilon\leq\epsilon'$, we have
$m_\mathrm{DMB}(\epsilon) \leq m_\eta(\epsilon)$.
Therefore,
\[
\liminf_{\epsilon\to 0} S(\epsilon)
~\geq~ \lim_{\epsilon\to 0} \frac{m_\mathrm{srl}(\epsilon)}{m_\eta(\epsilon)}
\frac{k}{1+\frac{\delta}{b(\epsilon)} k}
~=~\lim_{\epsilon\to 0}
\frac{\left(1+\sqrt{1+\frac{2L\epsilon}{\sigma^2}} \right)^2}{4(1+\eta)^2}
\frac{k}{1+\frac{\delta}{b(\epsilon)} k} ~=~ \frac{1}{(1+\eta)^2} k.
\]
Since the above inequality holds for any $\eta>0$, we can take $\eta\to 0$
and conclude that \linebreak[4] $\liminf_{\epsilon\to 0} S(\epsilon) \geq k$.
This finishes the proof.
\end{proof}

For accelerated stochastic gradient methods whose convergence rates
have a similar dependence on the gradient variance
\citep{Lan09, HuKwokPan09, Xiao10, GhadimiLan10},
the batch size~$b$ has a even smaller
effect on the convergence rate (see discussions after \thmref{thm:opt-DMB}),
which implies a better parallel speed-up.

\section{Experiments} \label{sec:experiments}

We conducted experiments with a large-scale online binary
classification problem. First, we obtained a log of one billion
queries issued to the Internet search engine Bing. Each entry in
the log specifies a time stamp, a query text, and the id of the user
who issued the query (using a temporary browser cookie). A query is
said to be \emph{highly monetizable} if, in the past, users who issued
this query tended to then click on online advertisements. Given a
predefined list of one million highly monetizable queries, we observe
the queries in the log one-by-one and attempt to predict whether the
next query will be highly monetizable or not. A clever search engine
could use this prediction to optimize the way it presents search
results to the user. A prediction algorithm for this task must keep up
with the stream of queries received by the search engine, which calls
for a distributed solution.

The predictions are made based on the recent query-history of the
current user. For example, the predictor may learn that users who
recently issued the queries ``island weather'' and ``sunscreen
reviews'' (both not highly monetizable in our data) are likely to issue
a subsequent query which is highly monetizable (say, a query like
``Hawaii vacation''). In the next section, we formally define how each
input, $z_t$, is constructed.

First, let $n$ denote the number of distinct queries that appear in
the log and assume that we have enumerated these queries, $q_1,\ldots,q_n$.
Now define $x_t \in \{0,1\}^n$ as follows
$$
x_{t,j} =
\begin{cases}
1&\text{if query $q_j$ was issued by the current user during the last
  two hours,}\\ 0&\text{otherwise.}
\end{cases}
$$
Let $y_t$ be a binary variable, defined as
$$
y_t =
\begin{cases}
+1&\text{if the current query is highly monetizable,}\\
-1&\text{otherwise.}
\end{cases}
$$
In other words, $y_t$ is the binary label that we are trying to
predict.  Before observing $x_t$ or $y_t$, our algorithm chooses a
vector $w_t \in \reals^n$. Then $x_t$ is observed and the
resulting binary prediction is the sign of their inner product
$\inner{w_t,x_t}$.  Next,
the correct label $y_t$ is revealed and our binary prediction is
incorrect if $y_t \inner{w_t,x_t} \leq 0$.  We can re-state this
prediction problem in an equivalent way by defining $z_t = y_t
x_t$, and saying that an incorrect prediction occurs when
$\inner{w_t, z_t} \leq 0$.

We adopt the logistic loss function as a smooth convex proxy to the
error indicator function. Formally, define~$f$ as
$$
f(w, z) ~=~ \log_2\big( 1 + \exp(-\inner{w, z}) \big) ~~.
$$
Additionally, we introduced the convex regularization constraint
$\norm{w_t}\leq C$, where $C$ is a predefined regularization parameter.

We ran the synchronous version of our distributed algorithm using the
Euclidean dual averaging update rule~(\ref{eqn:EuclideanDA})
in a cluster simulation.
The simulation allowed us to easily investigate the
effects of modifying the number of nodes in the cluster and the
latencies in the network.

We wanted to specify a realistic latency in our simulation, which
faithfully mimics the behavior of a real network in a search engine
datacenter.  To this end, we assumed that the nodes are connected via
a standard 1Gbs Ethernet network. Moreover, we assumed that the nodes
are arranged in a precomputed logical binary-tree communication
structure, and that all communication is done along the edges in this
tree.  We conservatively estimated the round-trip latency between
proximal nodes in the tree to be 0.5ms. Therefore, the total time to
complete each vector-sum network operation is $\log_2(k)$ ms, where $k$ is the
number of nodes in the cluster. We assumed that our search engine
receives $4$ queries per ms (which adds up to ten billion queries a
month). Overall, the number of queries discarded between mini-batches
is $\mu = 4 \log_2(k)$.

In all of our experiments, we use the algorithmic parameter
$\alpha_j=L+\gamma\sqrt{j}$ (see \thmref{thm:da-bound}).
We set the smoothness parameter~$L$ to a constant, and the parameter~$\gamma$
to a constant divided by $\sqrt{b}$.
This is because $L$ depends only on the loss function
$f$, which does not change in DMB, while $\gamma$ is proportional to
$\sigma$, the standard deviation of the gradient-averages.  We chose
the constants by manually exploring the parameter space on a separate
held-out set of 500 million queries.

We report all of our results in terms of the average loss suffered by the online algorithm.
This is simply defined as
$(1/t) \sum_{i=1}^t f(w_i, z_i)$. We cannot plot regret,
as we do not know the offline risk minimizer $w^\star$.

\subsection{Serial Mini-Batching}

\begin{figure}[t]
\centering
\psfrag{x}[cc]{number of inputs}
\psfrag{y}[bc]{average loss}
\includegraphics[width=0.6\textwidth]{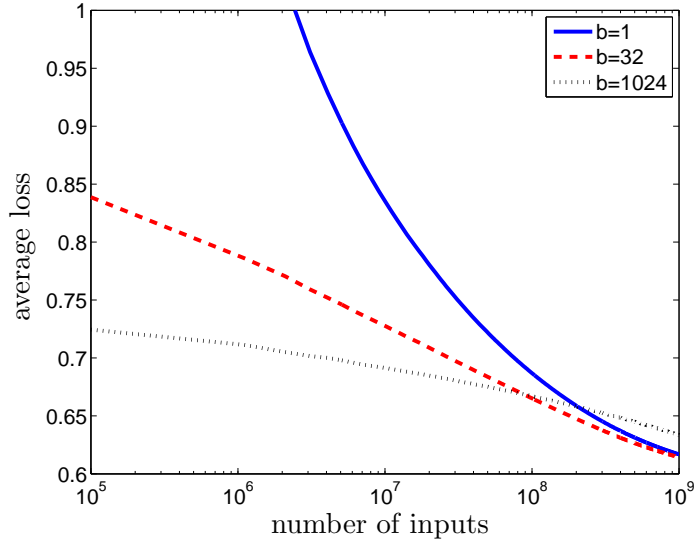}
\caption{ The effects of of the batch size when serial mini-batching on average loss. The mini-batches algorithm
was applied with different batch sizes. The x-axis presents the number of instances observed,
and the y-axis presents the average loss. Note that the case $b=1$ is the standard serial dual-averaging algorithm.}
\label{fig:minibatch}
\end{figure}

As a warm-up, we investigated the effects of modifying the mini-batch
size $b$ in a standard serial Euclidean dual averaging algorithm. This
is equivalent to running the distributed simulation with a cluster
size of $k=1$, with varying mini-batch size.  We ran the experiment
with $b = 1,2,4,\ldots,1024$. \figref{fig:minibatch}
shows the results for three representative mini-batch sizes. The
experiments tell an interesting story, which is more refined than our theoretical upper bounds. While the asymptotic worst-case theory implies that batch-size should have no significant effect, we actually observe that mini-batching accelerates the learning process on the first $10^8$ inputs. On the other hand, after $10^8$ inputs, a large mini-batch size begins to
hurt us and the smaller mini-batch sizes gain the lead. This behavior
is not an artifact of our choice of the parameters $\gamma$ and $L$,
as we observed a similar behavior for many different parameter
setting, during the initial stage when we tuned the parameters on a
held-out set.

Similar transient behaviors also exist for multi-step stochastic gradient
methods \citep[see, e.g.,][Section~4.3.2]{Polyak87}, where the multi-step
interpolation of the gradients also gives the smoothing effects as using
averaged gradients.
Typically such methods converge faster in the early iterations when the
iterates are far from the optimal solution and the relative value of the
stochastic noise is small, but become less effective asymptotically.

\subsection{Evaluating DBM}

\begin{figure}[t]
\begin{center}
\psfrag{x}[cc]{number of inputs}
\psfrag{y}[bc]{average loss}
\includegraphics[width=7cm]{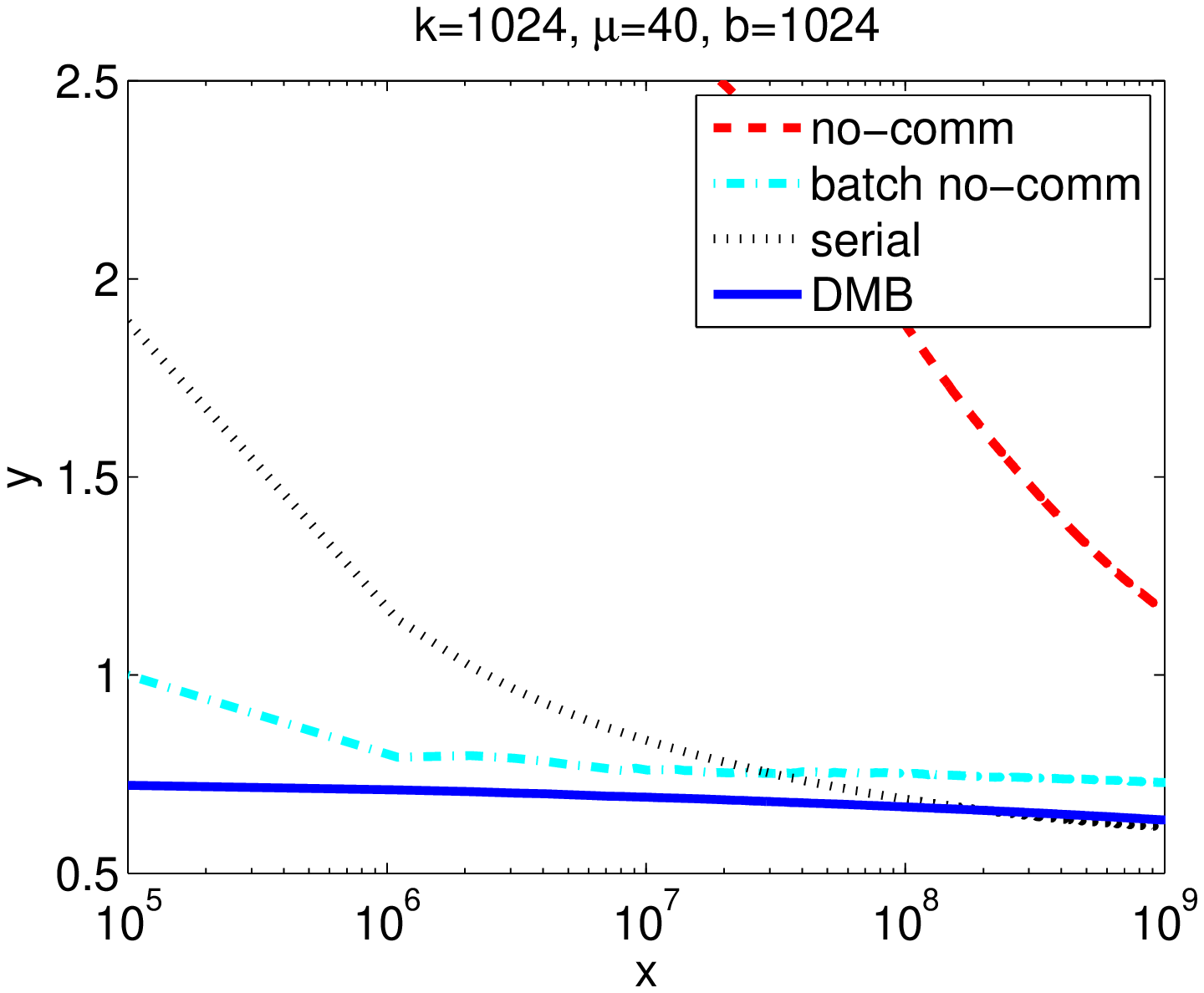}~~
\includegraphics[width=7cm]{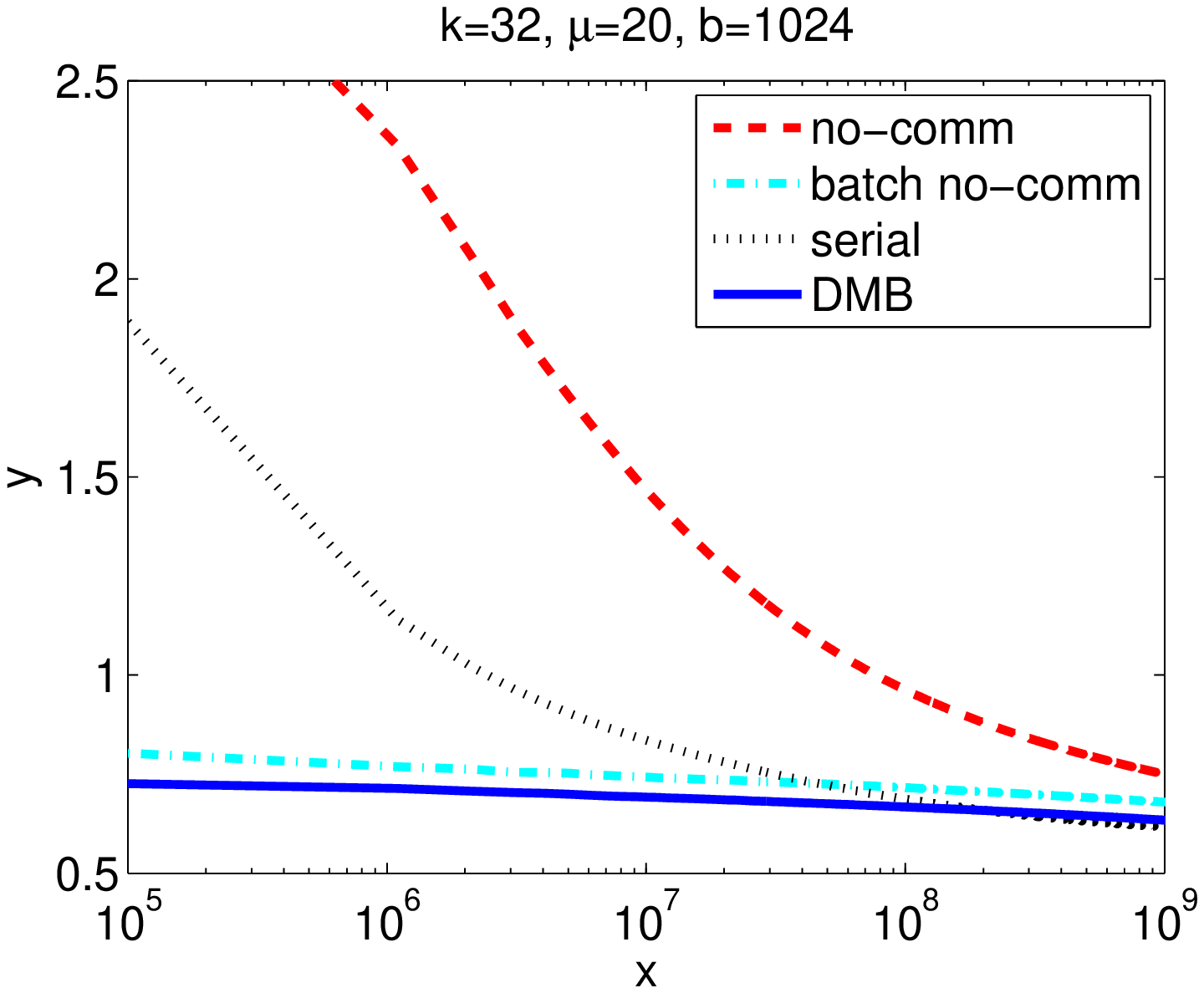}
\end{center}
\caption{ Comparing DBM with the serial algorithm and the
  no-communication distributed algorithm. Results for a large cluster of $k=1024$ machines are presented on the left. Results
for a small cluster of $k=32$ machines are presented on the right.}
\label{fig:dbm}
\end{figure}

Next, we compared the average loss of the DBM algorithm with the
average loss of the serial algorithm and the no-communication
algorithm (where each cluster node works independently).  We tried two
versions of the no-communication solution. The first version simply
runs $k$ independent copies of the serial prediction algorithm.  The
second version runs $k$ independent copies of the serial mini-batch
algorithm, with a mini-batch size of $128$.  We included the second
version of the no-communication algorithm after observing that
mini-batching has significant advantages even in the serial setting.
We experimented with various cluster sizes and various mini-batch
sizes. As mentioned above, we set the latency of the DBM algorithm to
$\mu = 4 \log_2(k)$. Taking a cue from our theoretical analysis, we
set the batch size to $b = m^{1/3} \simeq 1024$. We repeated the
experiment for various cluster sizes and the results were very
consistent. \figref{fig:dbm} presents the average loss of the three
algorithms for clusters of sizes $k=1024$ and $k=32$. Clearly, the
simple no-communication algorithm performs very
poorly compared to the others. The no-communication algorithm that uses
mini-batch updates on each node does surprisingly well,
but is still outperformed quite significantly by the DMB solution.

\subsection{The Effects of Latency}

\begin{figure}[t]
\begin{center}
\psfrag{x}[cc]{number of inputs}
\psfrag{y}[bc]{average loss}
\includegraphics[width=10cm]{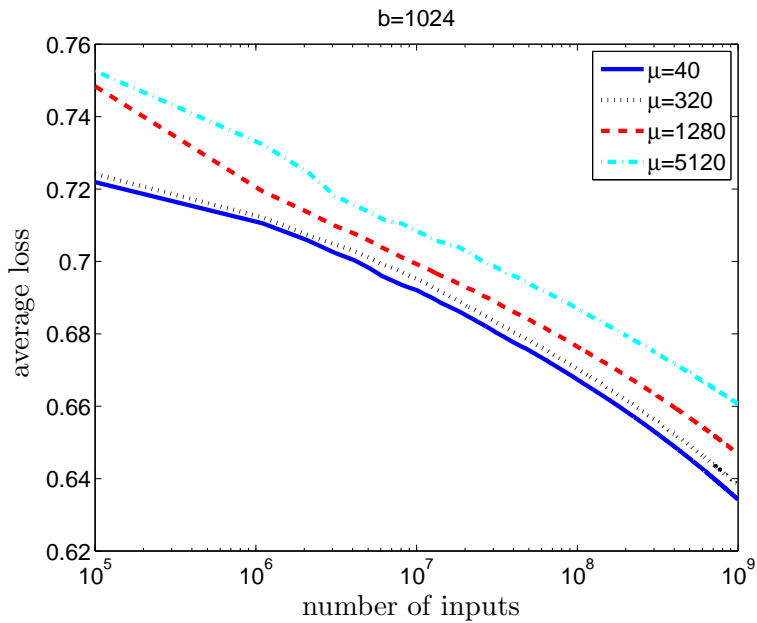}
\end{center}
\caption{The effects of increased network latency. The loss of the DMB algorithm is reported with different latencies as measured by $\mu$.
In all cases, the batch size is fixed at $b=1024$. }
\label{fig:latency}
\end{figure}

Network latency results in the DMB discarding gradients, and slows
down the algorithm's progress. The theoretical analysis shows that
this waste is negligible in the asymptotic worst-case sense. However,
latency will obviously have some negative effect on any finite prefix
of the input stream. We examined what would happen if the single-link
latency were much larger than our 0.5ms estimate (e.g., if the network
is very congested or if the cluster nodes are scattered across
multiple datacenters). Concretely, we set the cluster size to $k=1024$
nodes, the batch size to $b=1024$, and the single-link latency to
$0.5, 1, 2, \ldots, 512$ ms.  That is, $0.5$ms mimics a realistic 1Gbs
Ethernet link, while $512$ms mimics a network whose latency between
any two machines is 1024 times greater, namely, each vector-sum operation
takes a full second to complete. Note that $\mu$ is still computed as
before, namely, for latency $0.5 \cdot 2^p$, $\mu = 2^p4\log_2(k) =
2^p \cdot 40$.  \figref{fig:latency} shows how the average loss curve
reacts to four representative latencies. As expected, convergence rate
degrades monotonically with latency. When latency is set to be $8$
times greater than our realistic estimate for 1Gbs Ethernet, the
effect is minor.  When the latency is increased by a factor of $1024$,
the effect becomes more noticeable, but still quite small.

\subsection{Optimal Mini-Batch Size}

For our final experiment, we set out to find the optimal batch size
for our problem on a given cluster size.  Our theoretical analysis is
too crude to provide a sufficient answer to this question. The theory
basically says that setting $b = \Theta(m^\rho)$ is asymptotically optimal
for any $\rho \in (0,1/2)$, and that $b = \Theta(m^{1/3})$ is a pretty good
concrete choice. We have already seen that larger batch sizes
accelerate the initial learning phase, even in a serial setting.  We
set the cluster size to $k=32$ and set batch size to
$8,16,\ldots,4096$. Note that $b = 32$ is the case where each node
processes a single example before engaging in a vector-sum network operation.
\figref{fig:optb} depicts the average loss after $10^7, 10^8,$ and
$10^9$ inputs. As noted in the serial case, larger batch sizes
($b=512$) are beneficial at first ($m=10^7$), while smaller batch
sizes $(b=128$) are better in the end ($m=10^9$).

\begin{figure}[t]
\begin{center}
\psfrag{x}[cc]{$\log_2(b)$}
\includegraphics[width=5cm]{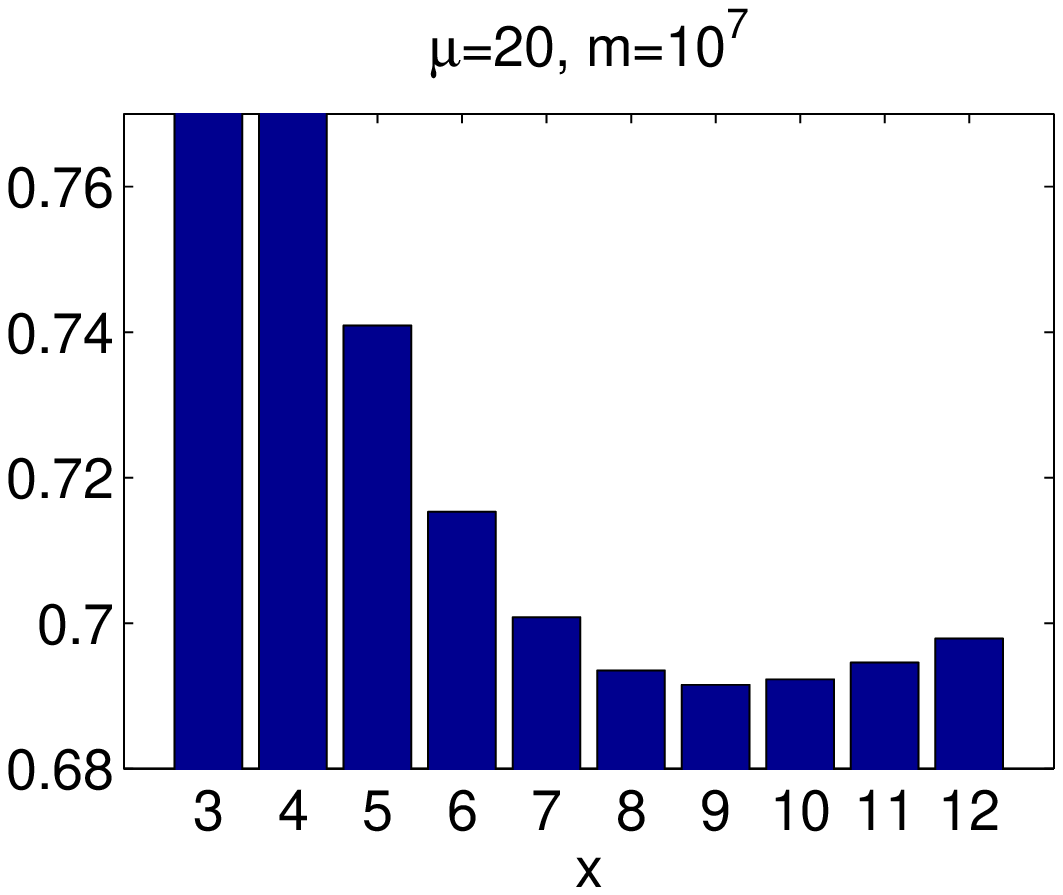}
\includegraphics[width=5cm]{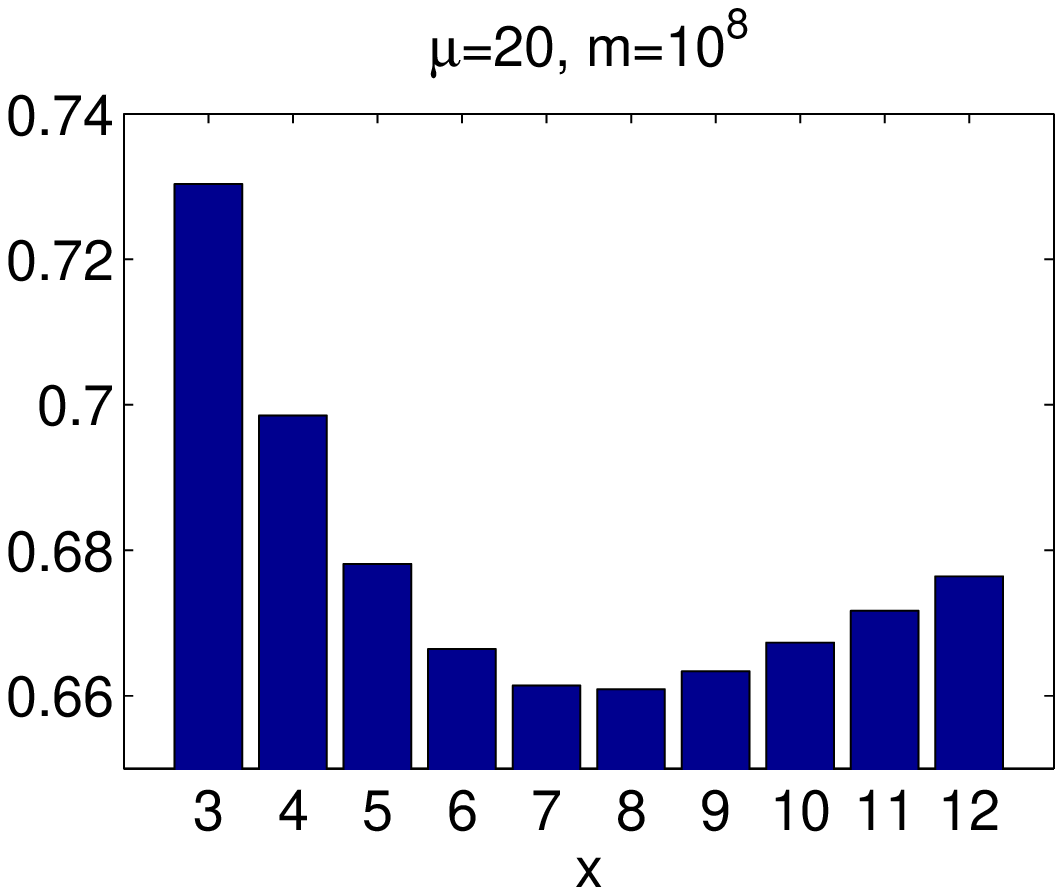}
\includegraphics[width=5cm]{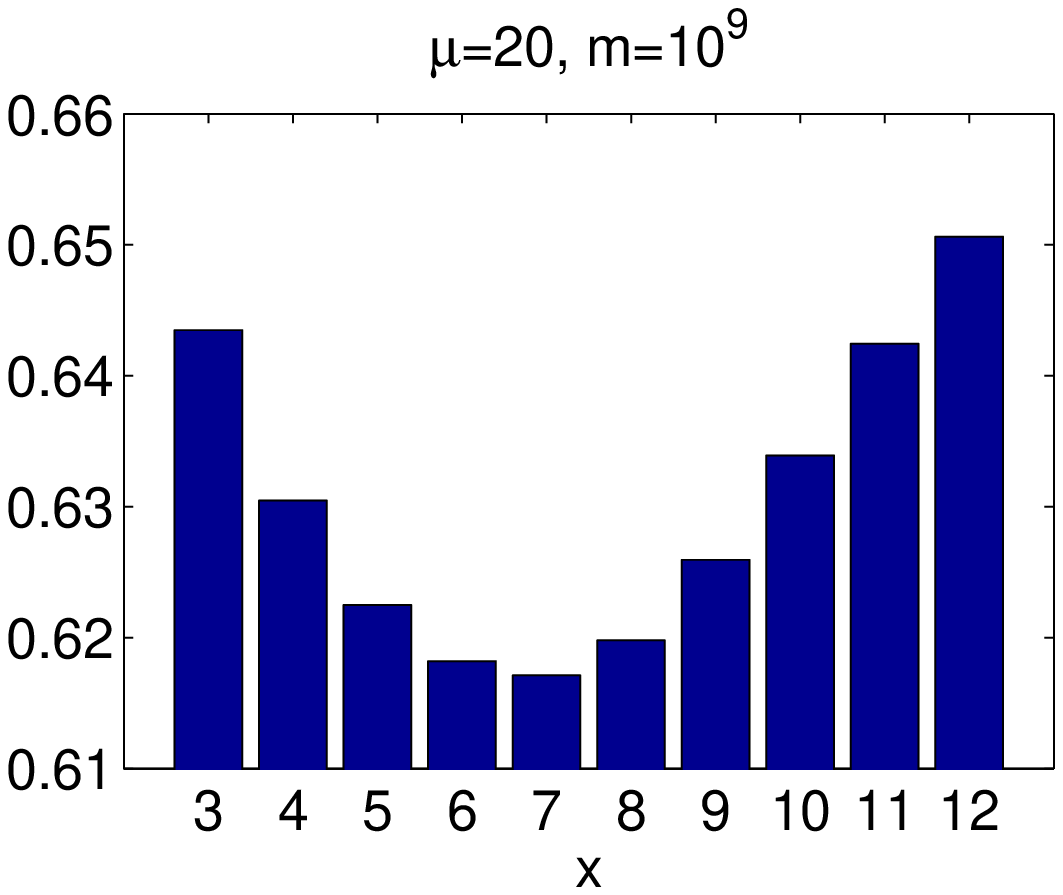}
\end{center}
\caption{ The effect of different mini-batch sizes ($b$) on the DBM algorithm. The DMB algorithm was applied with different batch sizes $b=8,\ldots,4096$.
 The loss is reported after $10^7$ instances (left), $10^8$ instances (middle) and $10^9$ instances (right).}
\label{fig:optb}
\end{figure}

\subsection{Discussion}
We presented an empirical evaluation of the serial mini-batch
algorithm and its distributed version, the DMB algorithm, on a
realistic web-scale online prediction problem. As expected, the DMB
algorithm outperforms the n\"aive no-communication algorithm. An
interesting and somewhat unexpected observation is the fact that the
use of large batches improves performance even in the serial setting.
Moreover, the optimal batch size seems to generally decrease with time.

We also demonstrated the effect of network latency on the performance
of the DMB algorithm. Even for relatively large values of $\mu$, the
degradation in performance was modest. This is an encouraging
indicator of the efficiency and robustness of the DMB algorithm, even
when implemented in a high-latency environment, such as a grid.

\section{Related Work}\label{sec:relatedwork}

In recent years there has been a growing interest in distributed online
learning and distributed optimization.

\cite{LangfordSmolaZinkevich09} address the distributed online
learning problem, with a similar motivation to ours: trying to address
the scalability problem of online learning algorithms which are
inherently sequential. The main observation
\cite{LangfordSmolaZinkevich09} make is that in many cases, computing
the gradient takes much longer than computing the update according to
the online prediction algorithm.  Therefore, they present a pipeline
computational model.  Each worker alternates between computing the
gradient and computing the update rule. The different workers are
synchronized such that no two workers perform an update
simultaneously.

Similar to results presented in this paper,
\cite{LangfordSmolaZinkevich09} attempted to show that it is possible
to achieve a cumulative regret of $O\left(\sqrt{m}\right)$ with $k$
parallel workers, compared to the $O\bigl(\sqrt{km}\bigr)$ of the
na\"ive solution. However their work suffers from a few
limitations.
First, their proofs only hold for unconstrained convex optimization
where no projection is needed.
Second, since they work in a model where one node at a time updates a shared
predictor, while the other nodes compute gradients, the scalability of
their proposed method is limited by the ratio between the time it
takes to compute a gradient to the time it takes to run the update
rule of the serial online learning algorithm.

In another related work, \cite{DuchiAgarwalWainwright10} present a
distributed dual averaging method for optimization over networks.
They assume the loss functions are Lipschitz continuous, but their
gradients may not be. Their method does not need synchronization to
average gradients computed at the same point. Instead, they employ a
distributed consensus algorithm on all the gradients generated by
different processors at different points.
When applied to the stochastic online prediction setting, even for
the most favorable class of communication graphs, with constant
spectral gaps (e.g., expander graphs), their best regret bound is
$O\bigl(\sqrt{km}\log(m)\bigr)$. This bound is no better than one
would get by running~$k$ parallel machines without communication (see
Section~\ref{sec:idealtrivial}).

In another recent work, \cite{ZinkevichWeimerSmolaLi10} study a method
where each node in the network runs the classic stochastic gradient
method, using random subsets of the overall data set, and only
aggregate their solutions in the end (by averaging their final weight
vectors).  In terms of online regret, it is obviously the same as
running~$k$ machines independently without communication.  So a more
suitable measure is the optimality gap (defined in
Section~\ref{sec:optimization}) of the final averaged predictor.  Even
with respect to this measure, their expected optimality gap does not
show advantage over running~$k$ machines independently. A similar
approach was also considered by \cite{NesterovVial08} and an
experimental study of such a method was reported in \cite{HHKPW03}.

A key difference between our DMB framework and many related work is
that DMB does not consider distributed comuting as a constraint to overcome.
Instead, our novel use of the variance-based regret bounds can exploit
parallel/distributed computing to obtain the asymptotic optimal regret bound.
Beyond the asymptotic optimality of our bounds, our work has other
features that set it apart from previous work. As far as we know, we
are the first to propose a general principled framework for
distributing many gradient-based update rule, with a concrete regret
analysis for the large family of mirror descent and dual averaging
update rules. Additionally, our work is the first to explicitly
include network latency in our regret analysis, and to theoretically
guarantee that a large latency can be overcome by setting parameters
appropriately.

\section{Conclusions and Further Research} \label{sec:conclusions}

The increase in serial computing power of modern computers is out-paced
by the growth rate of web-scale prediction problems and data sets.
Therefore, it is necessary to adopt techniques that can harness the
power of parallel and distributed computers.

In this work we studied the problems of distributed stochastic online
prediction and distributed stochastic optimization. We presented a
family of distributed online algorithms with asymptotically optimal
regret and optimality gap guarantees. Our algorithms use the
distributed computing infrastructure to reduce the variance of
stochastic gradients, which essentially reduces the noise in the
algorithm's updates. Our analysis shows that asymptotically, a
distributed computing system can perform as well as a hypothetical
fast serial computer. This result is far from trivial, and much of the
prior art in the field did not show any provable gain by using
distributed computers.

While the focus of this work is the theoretical analysis of a
distributed online prediction algorithm, we also presented experiments
on a large-scale real-world problem. Our experiments showed that
indeed the DMB algorithm outperforms other simple solutions. They also
suggested that improvements can be made by optimizing the batch size
and adjusting the learning rate based on empirical measures.

Our formal analysis hinges on the fact that the regret bounds of many
stochastic online update rules scale with the variance of the
stochastic gradients when the loss function is smooth. It is unclear
if smoothness is a necessary condition, or if it can be replaced with
a weaker assumption. In principle, our results apply in a broader
setting.  For any serial update rule $\phi$ with a regret bound
of $\psi(\sigma^2, m) = C \sigma \sqrt{m} + o\left(\sqrt{m}\right)$,
the DMB algorithm and its variants have the optimal regret bound of
$C \sigma \sqrt{m} + o\left(\sqrt{m}\right)$, provided that the bound
$\psi(\sigma^2,m)$ applies equally to the function~$f$
and to the function
\[
\bar f\left(w, \left(z_1,\ldots,z_b\right)\right)
~=~ \frac{1}{b} \sum_{s=1}^b f(w, z_s) ~.
\]
Note that this result holds independently of the network size~$k$ and
the network latency~$\mu$.  Extending our results to non-smooth
functions is an interesting open problem. A more ambitious challenge
is to extend our results to the non-stochastic case, where inputs may
be chosen by an adversary.


An important future direction is to develop distributed learning
algorithms that perform robustly and efficiently on heterogeneous
clusters and in asynchronous distributed environments. This direction has been further explored in \citet{DGSX11}. For example, one can use the following simple reformulation of the DMB algorithm in a master-workers setting: each worker process inputs at its own pace and periodically sends the accumulated gradients to the master; the master applies the update rule whenever the number of accumulated gradients reaches a certain threshold and
broadcasts the new predictor back to the workers. In a dynamic
environment, where the network can be partitioned and reconnected and
where nodes can be added and removed, a new master (or masters) can be
chosen as needed by a standard leader election algorithm. We refer the reader to \citet{DGSX11} for more details.

A central property of our method is that all of the gradients in a
batch must be taken at the same prediction point.
In an asynchronous distributed computing environment
\citep[see, e.g.,][]{TsitsiklisBerAth86,BertsekasTsitsiklis89}, this can be
quite wasteful. In order to reduce the waste generated by the need for
global synchronization, we may need to allow different nodes to
accumulate gradients at different yet close points.  Such a
modification is likely to work since the smoothness assumption
precisely states that gradients of nearby points are similar.
There have been extensive studies on distributed optimization with
inaccurate or delayed subgradient information, but mostly without the
smoothness assumption \citep[e.g.,][]{NedicBertsekasBorkar01,NedicOzdaglar09}.
We believe that our main results under the smoothness assumption can be
extended to asynchronous and distributed environments as well.

\acks{The authors are grateful to Karthik Sridharan for spotting the fact that
for non-Euclidean spaces, our variance-reduction argument had to be modified
with a space-dependent constant
(see the discussion at the end of Section~\ref{sec:dmb-algorithm}).
We also thank the reviewers for their helpful comments.
}

\appendix

\section{Smooth Stochastic Online Prediction in the Serial Setting}
\label{app:serial}

In this appendix, we prove expected regret bounds for stochastic dual
averaging and stochastic mirror descent applied to smooth loss
functions.
In the main body of the paper, we discussed only the
Euclidean special case of these algorithms, while here we present the
algorithms and regret bounds in their full generality.
In particular, \thmref{thm:gd-bound} is a special case of \thmref{t:md-regret},
and \thmref{thm:da-bound} is a special case of \thmref{t:da-regret}.

Recall that we observe a stochastic sequence of inputs
$z_1,z_2,\ldots$, where each $z_i \in \Zcal$. Before observing each
$z_i$ we predict $w_i \in W$, and suffer a loss $f(w_i,z_i)$. We
assume $W$ is a closed convex subset of a finite dimensional
vector space $\Vcal$ with endowed norm $\norm{\cdot}$. We assume that $f(w,z)$
is convex and differentiable in $w$, and we use $\nabla_w f(w,z)$ to
denote the gradient of $f$ with respect to its first
argument. $\nabla_w f(w,z)$ is a vector in the dual space $\Vcal^*$,
with endowed norm $\norm{\cdot}_*$.

We assume that $f(\cdot,z)$ is $L$-smooth for any realization
of $z$. Namely, we assume that $f(\cdot,z)$ is differentiable and that
$$
\forall\, z\in\Zcal, \quad \forall \, w,w' \in W,\qquad
\norm{\nabla_w f(w,z) - \nabla_w f(w',z)}_* \leq L \norm{w - w'} ~.
$$
We define $F(w) = \E_z [f(w,z)]$ and note that $\nabla_w
F(w)= \E_z [\nabla_w f(w,z)]$ (see \citealp{RockafellarWe82}). This
implies that
$$
\forall\, w,w' \in W, \qquad
\norm{\nabla_w F(w) - \nabla_w F(w')}_* \leq L \norm{w-w'}~.
$$
In addition, we assume that there exists a constant $\sigma \geq 0$ such that
$$
\forall\, w \in W, \qquad
\E_z[ \norm{\nabla_w f(w,z) - \nabla_w E_z[f(w,z)] }_*^2 ] \leq \sigma^2 ~.
$$
We assume that $w^\star = \argmin_{w \in W} F(w)$ exists, and
we abbreviate $F^\star = F(w^\star)$.

Under the above assumptions, we are concerned with bounding the expected regret $\E[R(m)]$,
where regret is defined as
$$
R(m) ~=~ \sum_{i=1}^m \left( f(w_i,z_i) - f(w^\star,z_i) \right) ~.
$$
In order to present the algorithms in their full generality, we first recall
the concepts of strongly convex function and Bregman divergence.

A function $h:W\to\reals\cup\{+\infty\}$ is said to be
$\mu$-\emph{strongly convex} with respect to $\norm{\cdot}$
if
$$
\forall \alpha \in [0,1], \quad \forall u,v \in W, \quad
h(\alpha u + (1-\alpha) v) \leq \alpha h(u) + (1-\alpha) h(v)
 - \frac{\mu}{2}\alpha(1-\alpha)\norm{u-v}^2 ~.
$$
If $h$ is $\mu$-strongly convex then for any $u \in \dom h$,
and $v\in\dom h$ that is sub-differentiable, then
$$
\forall s \in \partial h(v), \quad
h(u) \geq h(v) + \inner{s, u - v} + \frac{\mu}{2}\norm{u-v}^2~~.
$$
(See, e.g., \citealp{GoebelRo08}.)
If a function $h$ is
strictly convex and differentiable (on an open set contained
in $\dom h$), then we can defined the Bregman divergence generated by~$h$ as
$$
d_h(u,v) = h(u) - h(v) - \langle \nabla h(v),\,  u - v \rangle~.
$$
We often drop the subscript $h$ in $d_h$ when it is obvious from the context.
Some key properties of the Bregman divergence are:
\begin{itemize}
\item $d(u,v) \geq 0$, and the equality holds if and only if $u=v$.
\item In general $d(u,v)\neq d(v,u)$, and $d$ may not satisfy the
triangle inequality.
\item The following \emph{three-point identity} follows directly from
the definition:
$$
d(u,w) = d(u,v) + d(v,w) + \inner{ \nabla h(v) - \nabla h(w), u-v} ~~.
$$
\end{itemize}
The following inequality is a direct consequence of the $\mu$-strong convexity
of~$h$:
\begin{equation}\label{eqn:sc-norm-bound}
d(u,v) \geq \frac{\mu}{2} \|u-v\|^2 ~.
\end{equation}

\subsection{Stochastic Dual Averaging}
The proof techniques for the stochastic dual averaging method are
adapted from those for the accelerated algorithms presented in
\citet{Tseng08} and \citet{Xiao10}.

Let $h:W \to \reals$ be a $1$-strongly convex function.
Without loss of generality, we can assume that $\min_{w\in W} h(w) = 0$.
In the stochastic dual averaging method, we predict each $w_i$ by
\begin{equation}\label{e:stoch-da}
w_{i+1} = \argmin_{w \in W} \left\{ \left\langle \sum_{j=1}^{i} g_j,
w \right\rangle + (L + \beta_{i+1}) h(w) \right\} ~~,
\end{equation}
where $g_j$ denotes the stochastic gradient $\nabla_w f(w_j,z_j)$,
and $(\beta_i)_{i \geq 1}$ is a sequence of positive and nondecreasing
parameters (i.e., $\beta_{i+1} \geq \beta_{i}$). As a special case of the above,
we initialize $w_1$ to
\begin{equation}\label{e:da-ini}
w_1 = \argmin_{w\in W} h(w)~~.
\end{equation}

We are now ready to state a bound on the expected regret
of the dual averaging method, in the smooth stochastic case.
\begin{theorem}\label{t:da-regret}
The expected regret of the stochastic dual averaging method is bounded as
$$
\forall m, \quad \E[R(m)] \leq  (F(w_1)-F(w^\star))
+ (L+\beta_m) h(w^\star)
+ \frac{\sigma^2}{2} \sum_{i=1}^{m-1}\frac{1}{\beta_i}.
$$
\end{theorem}

The optimal choice of $\beta_i$
is exactly of order $\sqrt{i}$.  More specifically, let $\beta_i = \gamma \sqrt{i}$,
where~$\gamma$ is a positive parameter. Then \thmref{t:da-regret} implies that
\[
\E[R(m)] \leq (F(w_1)-F(w^\star)) + L h(w^\star) +
\left( \gamma h(w^\star) + \frac{\sigma^2}{\gamma} \right) \sqrt{m}.
\]
Choosing $\gamma=\sigma/\sqrt{h(w^\star)}$ gives
\[
\E[R(m)] \leq (F(w_1)-F(w^\star))+ L h(w^\star)
+ \left(2 \sigma\sqrt{h(w^\star)}\right) \sqrt{m}.
\]
If $\nabla F(w^\star)=0$ (this is certainly the case if $W$ is the whole space),
then we have
\[
F(w_1)-F(w^\star) \leq \frac{L}{2}\|w_1-w^\star\|^2 \leq L h(w^\star).
\]
Then the expected regret bound can be simplified as
\[
\E [R(m)] \leq 2L h(w^\star)
+ \left(2 \sigma\sqrt{h(w^\star)}\right) \sqrt{m}.
\]

To prove \thmref{t:da-regret} we require the following fundamental
lemma, which can be found, for example, in \cite{Nesterov05}, \cite{Tseng08}
and \cite{Xiao10}.
\begin{lemma}\label{l:composite-min}
Let~$W$ be a closed convex set, $\varphi$ be a convex function on~$W$,
and $h$ be $\mu$-strongly convex on~$W$ with respect to $\norm{\cdot}$.
If
\[
w^+ = \argmin_{w\in W} \bigl\{ \varphi(w) + h(w) \bigr\},
\]
then
\[
\forall\, w \in W, \qquad
\varphi(w) + h(w) \geq \varphi(w^+) + h(w^+) + \frac{\mu}{2} \|w-w^+\|^2.
\]
\end{lemma}

With \lemref{l:composite-min}, we are now ready to prove \thmref{t:da-regret}.
\begin{proof}
First, we define the linear functions
\[
\ell_i(w) = F(w_i) + \langle \nabla F(w_i), w-w_i \rangle, \qquad
\forall\, i\geq 1,
\]
and (using the notation $g_i=\nabla f(w_i,z_i)$)
\[
\hat \ell_i(w)
= F(w_i) + \langle g_i, w-w_i \rangle
= \ell_i(w) + \langle q_i, w-w_i \rangle ,
\]
where
\[
q_i = g_i - \nabla F(w_i).
\]
Therefore, the stochastic dual averaging method specified in \eqref{e:stoch-da} is
equivalent to
$$
w_i = \argmin_{w\in W} \left\{ \sum_{j=1}^{i-1} \hat \ell_j(w)
+ (L+\beta_i) h(w) \right\}.
$$

Using the smoothness assumption, we have (e.g., \citealt[Lemma~1.2.3]{Nesterov04})
\begin{eqnarray}
F(w_{i+1})
&\leq& \ell_i(w_{i+1}) + \frac{L}{2} \|w_{i+1}-w_i\|^2 \nonumber \\
&=& \hat \ell_i(w_{i+1}) + \frac{L+\beta_i}{2} \|w_{i+1}-w_i\|^2
   - \langle q_i, w_{i+1}-w_i\rangle - \frac{\beta_i}{2}\|w_{i+1}-w_i\|^2
   \nonumber\\
&\leq & \hat \ell_i(w_{i+1}) + \frac{L+\beta_i}{2} \|w_{i+1}-w_i\|^2
   + \|q_i\|_* \|w_{i+1}-w_i\| - \frac{\beta_i}{2}\|w_{i+1}-w_i\|^2
   \nonumber\\
&=&  \hat \ell_i(w_{i+1}) + \frac{L+\beta_i}{2} \|w_{i+1}-w_i\|^2
  -\left(\frac{1}{\sqrt{2\beta_i}}\|q_i\|_* - \sqrt{\frac{\beta_i}{2}}
 \|w_{i+1}-w_i\|\right)^2 +\frac{\|q_i\|_*^2}{2\beta_i} \nonumber\\
&\leq& \hat \ell_i(w_{i+1}) + \frac{L+\beta_i}{2} \|w_{i+1}-w_i\|^2
  +\frac{\|q_i\|_*^2}{2\beta_i}. \label{e:psi-upperbound}
\end{eqnarray}
Next we use \lemref{l:composite-min} with
$\varphi(w)=\sum_{j=1}^{i-1} \hat \ell_j(w)$ and $\mu=(L+\beta_i)$,
\[
\sum_{j=1}^{i-1} \hat \ell_j(w_{i+1}) + (L+\beta_i) h(w_{i+1})
\geq \sum_{j=1}^{i-1} \hat \ell_j(w_i) + (L+\beta_i) h(w_i)
+ \frac{L+\beta_i}{2}\|w_{i+1}-w_i\|^2,
\]
Combining the above inequality with \eqref{e:psi-upperbound},
we have
\begin{eqnarray*}
F(w_{i+1})
&\leq& \hat \ell_i(w_{i+1}) + \sum_{j=1}^{i-1} \hat \ell_j(w_{i+1})
  + (L+\beta_i) h(w_{i+1}) - \sum_{j=1}^{i-1} \hat \ell_j(w_i)
  - (L+\beta_i) h(w_i)+\frac{\|q_i\|_*^2}{2\beta_i}\\
&\leq&  \sum_{j=1}^i \hat \ell_j(w_{i+1}) + (L+\beta_{i+1}) h(w_{i+1})
    -\sum_{j=1}^{i-1} \hat \ell_j(w_i) - (L+\beta_i) h(w_i)
  +\frac{\|q_i\|_*^2}{2\beta_i},
\end{eqnarray*}
where in the last inequality, we used the assumptions
$\beta_{i+1}>\beta_i>0$ and $h(w_{i+1})\geq0$.
Summing the above inequality from $i=1$ to $i=m-1$,
we have
\begin{eqnarray*}
\sum_{i=2}^m F(w_i)
&\leq& \sum_{i=1}^{m-1} \hat \ell_i(w_m) + (L+\beta_m) h(w_m)
      +\sum_{i=1}^{m-1}\frac{\|q_i\|_*^2}{2\beta_i} \\
&\leq& \sum_{i=1}^{m-1} \hat \ell_i(w^\star) + (L+\beta_m) h(w^\star)
      +\sum_{i=1}^{m-1}\frac{\|q_i\|_*^2}{2\beta_i} \\
&\leq& \sum_{i=1}^{m-1} \ell_i(w^\star) + (L+\beta_m) h(w^\star)
      +\sum_{i=1}^{m-1}\frac{\|q_i\|_*^2}{2\beta_i}
      +\sum_{i=1}^{m-1} \langle q_i, w^\star - w_i \rangle\\
&\leq& (m-1) F(w^\star) + (L+\beta_i) h(w^\star)
      +\sum_{i=1}^{m-1}\frac{\|q_i\|_*^2}{2\beta_i}
      +\sum_{i=1}^{m-1} \langle q_i, w^\star - w_i \rangle.
\end{eqnarray*}
Therefore,
\begin{equation}\label{e:sumFbd}
\sum_{i=2}^m \bigl( F(w_i) - F(w^\star) \bigr)
\leq (L+\beta_m) h(w^\star)
+\sum_{i=1}^{m-1}\frac{\|q_i\|_*^2}{2\beta_i}
+\sum_{i=1}^{m-1} \langle q_i, w^\star - w_i \rangle.
\end{equation}
Notice that each $w_i$ is a deterministic function of
$z_1,\ldots,z_{i-1}$, so
\[
\E_{z_i} \bigl(\langle q_i,w^\star-w_i\rangle \,|\,
z_1,\ldots,z_{i-1} \bigr)=0
\]
by recalling the definition $q_i=\nabla f(w_i,z_i)-\nabla F(w_i)$.
Taking expectation of both sides of \eqref{e:sumFbd}
with respect to $z_1,\ldots,z_m$,
and adding the term $F(w_1)-F(w^\star)$,
we have
\[
\E \sum_{i=1}^m \bigl( F(w_i) - F(w^\star) \bigr)
\leq  F(w_1) - F(w^\star) + (L+\beta_m) h(w^\star)
+\sum_{i=1}^{m-1}\frac{\sigma^2}{2\beta_i} .
\]
Theorem~\ref{t:da-regret} is proved by further noticing
\[
\E\, f(w_i,z_i) = \E\, F(w_i), \qquad
\E\, f(w^\star,z_i) = F(w^\star),
\qquad \forall\, i\geq 1,
\]
which are due to the fact that $w_i$ is a deterministic function of
$z_0,\ldots,z_{i-1}$.
%
\end{proof}

\subsection{Stochastic Mirror Descent}
Variance-based convergence rates for the stochastic Mirror Descent methods
are due to \citet{JuditskyNT11},
and were extended to an accelerated
stochastic Mirror Descent method by \citet{Lan09}.
For completeness, we adapt their proofs to the context of regret
for online prediction problems.

Again let $h:W \to \reals$ be a differentiable $1$-strongly convex
function with $\min_{w\in W} h(w) = 0$.  Also let $d$ be the
Bregman divergence generated by~$h$.  In the stochastic mirror descent
method, we use the same initialization as in the dual averaging
method~(see \eqref{e:da-ini}) and then we set
$$
w_{i+1} = \argmin_{w \in W} \Bigl\{ \langle g_{i}, w \rangle + (L+\beta_{i})
d(w,w_{i}) \Bigr\}, \qquad i\geq 1.
$$
As in the dual averaging method, we assume that the sequence
$(\beta_i)_{i\geq 1}$ to be positive and nondecreasing.

\begin{theorem}\label{t:md-regret}
Assume that the convex set~$W$ is closed and bounded.
In addition assume $d(u,v)$ is bounded on~$W$ and let
\[
D^2 = \max_{u,v\in W} d(u,v).
\]
Then the expected regret of the stochastic mirror descent method is bounded as
\[
\E[R(m)] \leq (F(w_1)-F(w^\star)) + (L+\beta_m) D^2
+ \frac{\sigma^2}{2} \sum_{i=1}^{m-1}\frac{1}{\beta_i}.
\]
\end{theorem}

Similar to the dual averaging case, using the sequence of parameters
$\beta_i = (\sigma/D)\sqrt{i}$
gives the expected regret bound
\[
\E [R(m)]
\leq (F(w_1)-F(w^\star)) + L D^2 +  \left(2 \sigma D \right) \sqrt{m}.
\]
Again, if $\nabla F(w^\star)=0$, we have
$F(w_1)-F(w^\star)\leq (L/2)\|w_1-w^\star\|^2\leq L h(w^\star) \leq L D^2$,
thus the simplified bound
\[
\E [R(m)] \leq 2L D^2 +  \left(2 \sigma D \right) \sqrt{m}.
\]

We note that here we have stronger assumptions than in the dual
averaging case.  These assumptions are certainly satisfied by using
the standard Euclidean distance $d(u,v)=(1/2)\|u-v\|_2^2$ on a compact
convex set~$W$.  However, it excludes the case of using the
KL-divergence $d(u,v)=\sum_{i=1}^n u_i\log(u_i/v_i)$ on the
simplex, because the KL-divergence is unbounded on the simplex.
Nevertheless, it is possible to remove such restrictions by
considering other variants of the stochastic mirror descent method.
For example, if we use a constant $\beta_i$ that depends on the prior
knowledge of the number of total steps to be performed, then we can
weaken the assumption and replace~$D$ in the above bounds by
$\sqrt{h(w^\star)}$.  More precisely, we have
\begin{theorem}\label{t:md-regret-const}
Suppose we know the total number of steps~$m$ to be performed by the
stochastic mirror descent method ahead of time.
Then by using the initialization in \eqref{e:da-ini} and the
constant parameter
\[
\beta_i = \frac{\sigma}{\sqrt{2 h(w^\star)}}\sqrt{m} ,
\]
we have the expected regret bound
\[
\E [R(m)] \leq (F(w_1)-F(w^\star)) + L h(w^\star) + \sigma\sqrt{2 h(w^\star)} \sqrt{m}.
\]
\end{theorem}
\thmref{t:md-regret-const} is essentially the same as a result in~\citet{Lan09}, who also
developed an accelerated versions of the stochastic mirror descent method.
To prove \thmref{t:md-regret} and \thmref{t:md-regret-const} we need the
following standard Lemma,
which can be found in \citet{ChenTe93}, \citet{LanLuMo06} and \citet{Tseng08}.

\begin{lemma}\label{l:Bregman-reg}
Let~$W$ be a closed convex set,
$\varphi$ be a convex function on~$W$,
and $h$ be a differentiable, strongly convex function on~$W$.
Let $d$ be the Bregman divergence generated by~$h$.
Given $u\in W$, if
\[
w^+ = \argmin_{w\in W} \, \bigl\{ \varphi(w) + d(w,u) \bigr\},
\]
then
$$
\varphi(w)+d(w,u) \geq \varphi(w^+) + d(w^+,u) + d(w,w^+).
$$
\end{lemma}

We are ready to prove \thmref{t:md-regret} and \thmref{t:md-regret-const}.
\begin{proof}
We start with the inequality in \eqref{e:psi-upperbound}.
Using \eqref{eqn:sc-norm-bound} with $\mu=1$ gives
\begin{equation}\label{e:psi-upperbound2}
F(w_{i+1}) \leq \hat\ell_i(w_{i+1}) + (L+\beta_i) d(w_{i+1},w_i)
+\frac{\|q_i\|_*^2}{2\beta_i}.
\end{equation}
Now using \lemref{l:Bregman-reg} with $\varphi(w)=\hat\ell_i(w)$ yields
\[
\hat\ell_i(w_{i+1}) + (L+\beta_i) d(w_{i+1},w_i) \leq
\hat\ell_i(w^\star) + (L+\beta_i) d(w^\star,w_i)-(L+\beta_i) d(w^\star,w_{i+1}).
\]
Combining with \eqref{e:psi-upperbound2} gives
\begin{align*}
&F(w_{i+1})
~\leq~ \hat\ell_i(w^\star)
+ (L+\beta_i) d(w^\star,w_i)-(L+\beta_i) d(w^\star,w_{i+1})
+\frac{\|q_i\|_*^2}{2\beta_i}\\
&=~ \ell_i(w^\star)
+ (L+\beta_i) d(w^\star,w_i)-(L+\beta_{i+1}) d(w^\star,w_{i+1})
+(\beta_{i+1}-\beta_i) d(w^\star, w_{i+1}) \\
& \qquad +\frac{\|q_i\|_*^2}{2\beta_i} + \langle q_i, w^\star-w_i \rangle \\
&\leq~ F(w^\star)
+ (L+\beta_i) d(w^\star,w_i)-(L+\beta_{i+1}) d(w^\star,w_{i+1})
+(\beta_{i+1}-\beta_i) D^2\\
&\qquad +\frac{\|q_i\|_*^2}{2\beta_i} + \langle q_i, w^\star-w_i \rangle,
\end{align*}
where in the last inequality, we used the definition of $D^2$ and the
assumption that $\beta_{i+1}\geq \beta_i$.
Summing the above inequality from $i=1$ to $i=m-1$, we have
\begin{eqnarray*}
\sum_{i=2}^m F(w_i)
&\leq& (m-1)F(w^\star) + (L+\beta_1)d(w^\star,w_1) - (L+\beta_m)d(w^\star,w_m)
+(\beta_m-\beta_1)D^2 \\
&& \qquad + \sum_{i=1}^{m-1}\frac{\|q_i\|_*^2}{2\beta_i}
+ \sum_{i=1}^{m-1} \langle q_i, w^\star-w_i \rangle.
\end{eqnarray*}
Notice that $d(w^\star, w_i)\geq 0$ and $d(w^\star,w_1)\leq D^2$, so we have
\[
\sum_{i=2}^m F(w_i)
\leq (m-1)F(w^\star) + (L +\beta_m) D^2
+ \sum_{i=1}^{m-1}\frac{\|q_i\|_*^2}{2\beta_i}
+ \sum_{i=1}^{m-1} \langle q_i, w^\star-w_i \rangle.
\]
The rest of the proof for~\thmref{t:md-regret} is similar to that
for the dual averaging method (see arguments following \eqref{e:sumFbd}).

Finally we prove \thmref{t:md-regret-const}.
From the proof of \thmref{t:md-regret} above,
we see that if $\beta_i=\beta_m$ is a constant for all~$i=1,\ldots,m$,
then we have
\[
\E \sum_{i=2}^m (F(w_i) - F(w^\star))
\leq (L +\beta_m) d(w^\star, w_1)
+ \frac{\sigma^2}{2}\sum_{i=1}^{m-1}\frac{1}{\beta_i}.
\]
Notice that for the above result, we do not need to assume
boundedness of~$W$, nor boundedness of the Bregman divergence
$d(u,v)$.
Since we use $w_1=\argmin_{w\in W} h(w)$ and assume $h(w_1)=0$
(without loss of generality), it follows
$d(w^\star,w_1) \leq h(w^\star)$.
Plugging in $\beta_m = (\sigma/\sqrt{2 h(w^\star)})\sqrt{m}$ gives
the desired result.
\end{proof}

\section{High-Probability Bounds}\label{app:highprob}

For simplicity, the theorems stated throughout the paper involved bounds on the
expected regret, $\E[R(m)]$. A stronger type of result is a high-probability
bound, where $R(m)$ itself is bounded with arbitrarily high probability
$1-\delta$, and the bound having only logarithmic dependence on $\delta$. Here,
we demonstrate how our theorems can be extended to such high-probability
bounds.

First, we need to justify that the expected regret bounds for the online
prediction rules discussed in \appref{app:serial} have high-probability
versions. For simplicity, we will focus on a high-probability version of the
regret bound for dual averaging (\thmref{t:da-regret}),  but exactly the same
technique will work for stochastic mirror descent (\thmref{t:md-regret} and
\thmref{t:md-regret-const}).
With these results in hand, we will show how our main theorem for distributed
learning using the DMB algorithm (\thmref{thm:synchronous}) can be extended to
a high-probability version. Identical techniques will work for the other
theorems presented in the paper.

Before we begin, we will need to make a few additional mild
assumptions.
First, we assume that there are
positive constants $B,G$ such that $|f(w,z)|\leq B$ and
$\norm{\nabla_w f(w,z)}\leq G$ for all $w\in W$ and $z\in
\Zcal$. Second, we assume that there is a positive constant
$\hat{\sigma}$ such that $\Var_z(f(w,z)-f(w^\star,z))\leq
\hat{\sigma}^2$ for all $w\in W$ (note that $\hat{\sigma}^2
\leq 4B^2$ always holds). Third, that $W$ has a bounded
diameter $D$, namely $\norm{w-w'}\leq D$ for all $w,w'\in W$.

Under these assumptions, we can show the following high-probability version of
\thmref{t:da-regret}.

\begin{theorem}\label{t:da-regret-highprob}
For any $m$ and any $\delta\in (0,1]$, the regret of the stochastic dual averaging method is bounded with probability at least $1-\delta$ over the sampling of $z_1,\ldots,z_m$ by
\begin{align*}
R(m) \leq  (F(w_1)&-F(w^\star))
+ (L+\beta_m) h(w^\star)
+ \frac{\sigma^2}{2} \sum_{i=1}^{m-1}\frac{1}{\beta_i}\\
&+2\log(2/\delta)\left(DG+\frac{2G^2}{\beta_1}\right)
\sqrt{1+36\frac{G^2\sigma^2 \sum_{i=1}^{m}\frac{1}{\beta_i^2}+D^2\sigma^2m}{\log(2/\delta)}}\\
&+4\log(2/\delta)B\sqrt{1+\frac{18m\hat{\sigma}^2}{\log(2/\delta)}}.
\end{align*}
\end{theorem}

\begin{proof}
The proof of the theorem is identical to the one of \thmref{t:da-regret}, up to \eqref{e:sumFbd}:
\begin{equation}\label{e:sumFbd_highprob}
\sum_{i=2}^m \bigl( F(w_i) - F(w^\star) \bigr)
\leq (L+\beta_m) h(w^\star)
+\sum_{i=1}^{m-1}\frac{\|q_i\|^2}{2\beta_i}
+\sum_{i=1}^{m-1} \langle q_i, w^\star - w_i \rangle.
\end{equation}
In the proof of \thmref{t:da-regret}, we proceeded by taking expectations of both sides with respect to the sequence $z_1,\ldots,z_m$. Here, we will do things a bit differently.

The main technical tool we use is a well-known Bernstein-type inequality for martingales \citep[e.g.,][Lemma~A.8]{CesaBianchiLu06}, an immediate corollary of which can be stated as follows: suppose $x_1,\ldots,x_m$ is a martingale difference sequence with respect to the sequence $z_1,\ldots,z_m$, such that $|x_i|\leq b$, and let
\[
v = \sum_{i=1}^{m}\text{Var}(x_i|z_1,\ldots,z_{i-1}).
\]
Then for any $\delta\in (0,1)$, it holds with probability at least $1-\delta$ that
\begin{equation}\label{e:freedman}
\sum_{i=1}^{m}x_i \leq
b\log(1/\delta)\sqrt{1+\frac{18v}{\log(1/\delta)}}.
\end{equation}

Recall the definition $q_i=\nabla f(w_i,z_i)-\nabla F(w_i)$,
and let $\sigma_i^2 = \E[\norm{q_i}^2]$.
Note that $\sigma_i^2\leq \sigma^2$.
We will first use this result for the sequence
\[
x_i=\frac{\norm{q_i}^2-\sigma_i^2}{2\beta_i}+\inner{q_i,w^\star-w_i}.
\]
It is easily seen that $\E_{z_i}[x_i|z_1,\ldots,z_{i-1}]=0$, so it is indeed a martingale difference sequence w.r.t. $z_1,\ldots,z_m$. Moreover, $|\inner{q_i,w^\star-w_i}|\leq D\norm{q_i}\leq 2DG$, $\norm{q_i}^2\leq 4G^2$. In terms of the variances, let $\text{Var}_{z_i}$ and $\E_{z_i}$ be shorthand for the variance (resp.\ expectation) over $z_i$ conditioned over $z_1,\ldots,z_{i-1}$. Then
\begin{align*}
\text{Var}_{z_i}(x_i)&\leq 2\text{Var}_{z_i}\left(\frac{\norm{q_i}^2-\sigma_i^2}{2\beta_i}\right)+
2\text{Var}_{z_i}\left(\inner{q_i,w^\star-w_i}\right)\\
&\leq \frac{1}{2}\E_{z_i}\left(\frac{\norm{q_i}^4}{\beta_i^2}\right)+2\E_{z_i}[(\inner{q_i,w^\star-w_i})^2]\\ &\leq 2G^2\E_{z_i}\left(\frac{\norm{q_i}^2}{\beta_i^2}\right)+2\norm{w^\star-w_i}^2\E_{z_i}[\norm{q_i}^2]\\
&\leq 2G^2\frac{\sigma_i^2}{\beta_i^2}+2D^2\sigma_i^2 ~\leq~ 2G^2\frac{\sigma^2}{\beta_i^2}+2D^2\sigma^2.
\end{align*}
Combining these observations with \eqref{e:freedman}, we get that with probability at least $1-\delta$,
\begin{equation}\label{e:freedman1}
\sum_{i=1}^{m-1} \frac{\norm{q_i}^2-\sigma^2}{\beta_i}+\inner{q_i, w^\star - w_i} \leq
\left(2DG+\frac{4G^2}{\beta_1}\right)\log(1/\delta)\sqrt{1+36\frac{G^2\sigma^2 \sum_{i=1}^{m}\frac{1}{\beta_i^2}+D^2\sigma^2m}{\log(1/\delta)}}.
\end{equation}

A similar type of bound can be derived for the sequence $x_i=\left(f(w_i,z_i)-f(w^\star,z_i)\right)-$ \linebreak[4] $\left(F(w_i)-F(w^\star)\right)$. It is easily verified to be a martingale difference sequence w.r.t. $z_1,\ldots,z_m$, since
\[
\E\left[\left(f(w_i,z_i)-f(w^\star,z_i)\right)-\left(F(w_i)-F(w^\star)\right)
|z_1,\ldots,z_{i-1}\right]=0.
\]
Also,
\[
\left|\left(f(w_i,z_i)-f(w^\star,z_i)\right)-\left(F(w_i)-F(w^\star)\right)\right|
\leq 4B,
\]
and
\begin{align*}
\text{Var}_{z_i}\left(\bigl(f(w_i,z_i)-f(w^\star,z_i)\bigr)-\bigl(F(w_i)-F(w^\star)\bigr)\right)
&\,=\, \text{Var}_{z_i}\bigl(f(w_i,z_i)-f(w^\star,z_i)\bigr) \\
&\,\leq\, \hat{\sigma}^2~.
\end{align*}
So again using \eqref{e:freedman}, we have that with probability at least $1-\delta$ that
\begin{equation}\label{e:freedman2}
\sum_{i=1}^{m} \left(f(w_i,z_i)-f(w^\star,z_i)\right)-\left(F(w_i)-F(w^\star)\right)
\leq 4B\log(1/\delta)\sqrt{1+\frac{18m\hat{\sigma}^2}{\log(1/\delta)}}~~.
\end{equation}

Finally, adding $F(w_1)-F(w^\star)$ to both sides of \eqref{e:sumFbd_highprob}, and combining \eqref{e:freedman1} and \eqref{e:freedman2} with a union bound, the result follows.
\end{proof}

Comparing the theorem to \thmref{t:da-regret}, and assuming that $\beta_i = \Theta(\sqrt{i})$, we see that the bound has additional $O(\sqrt{m})$ terms. However, the bound retains the important property of having the dominant terms multiplied by the variances $\sigma^2,\hat{\sigma}^2$. Both variances become smaller in the mini-batch setting, where the update rules are applied over averages of $b$ such functions and their gradients. As we did earlier in the paper, let us think of this bound as an abstract function $\psi(\sigma^2,\hat{\sigma}^2,\delta,m)$. Notice that now, the regret bound also depends on the function variance $\hat{\sigma}^2$, and the confidence parameter $\delta$.

\begin{theorem} \label{thm:synchronous-highprob}
Let~$f$ is an $L$-smooth convex loss function.
Assume that the stochastic gradient $\nabla_w f(w,z_i)$ is bounded by a constant
and has $\sigma^2$-bounded variance for all~$i$ and all~$w$,
and that $f(w,z_i)$ is bounded by a constant and has $\hat{\sigma}^2$-bounded
variance for all~$i$ and for all~$w$.
If the update rule~$\phi$ has a serial high-probability regret bound
$\psi(\sigma^2,\hat{\sigma}^2,\delta,m)$.
then with probability at least $1-\delta$, the total regret of
\algref{alg:distributed} over $m$ examples is at most
$$
(b+\mu)\psi\left(\frac{\sigma^2}{b},\frac{\hat{\sigma}^2}{b},\delta,1 + \frac{m}{b+\mu}\right)+
O\left(\hat{\sigma}\sqrt{\left(1+\frac{\mu}{b}\right)\log(1/\delta)m}\right) ~~.
$$
\end{theorem}

Comparing the obtained bound to the one in \thmref{thm:synchronous}, we note that we pay an additional $O(\sqrt{m})$ factor.

\begin{proof}
The proof closely resembles the one of \thmref{thm:synchronous}. We let $\bar z_j$ denote the first $b$ inputs on batch $j$,
and define $\bar f$ as the average loss on these inputs. Note that for any $w$, the variance of $\bar f(w,\bar z_j)$ is at most $\hat{\sigma}^2/b$, and the variance of $\nabla_{w} \bar f(w,z)$ is at most $\sigma^2/b$. Therefore, with probability at least $1-\delta$, it holds that
\begin{equation}
\sum_{j=1}^{\bar m} \left( \bar f(w_j,\bar z_j) - \bar f( w^\star, \bar z_j) \right)
\leq \psi\left(\frac{\sigma^2}{b},\frac{\hat{\sigma}^2}{b},\delta,\bar{m}\right)~~.
\label{eqn:dist1-highprob}
\end{equation}
where $\bar m$ is the number of inputs given to the update rule $\phi$. Let $Z_j$ denote the set of all examples received between the commencement of batch $j$ and the commencement of batch $j+1$, including the vector-sum phase in between ($b+\mu$ examples overall). In the proof of \thmref{thm:synchronous}, we had that
$$
\E\left[ \left( \bar f(w_j,\bar z_j) - \bar f( w^\star, \bar z_j) \right)
\,|\, w_j \right] ~=~
\E\bigg[ \frac{1}{b+\mu} \sum_{z\in Z_j} \left( f(w_j, z_i) - f( w^\star, z_i)
\right)  ~\Big|~ w_j \bigg]~~,
$$
and thus the \emph{expected value} of the left-hand side of
\eqref{eqn:dist1-highprob} equals the total regret, divided by $b + \mu$. Here, we need to work a bit harder. To do so, note that the sequence of random variables
\[
\bigg(\frac{1}{b}\sum_{z\in \bar z_j}\bigl(f(w_j,z)-f(w^\star,z)\bigr)\bigg)-
\bigg(\frac{1}{b+\mu}\sum_{z\in Z_j}\bigl(f(w_j,z)-f(w^\star,z)\bigr)\bigg),
\]
indexed by $j$, is a martingale difference sequence with respect to $Z_1,Z_2,\ldots$. Moreover, conditioned on $Z_1,\ldots,Z_{j-1}$, the variance of each such random variable is at most $4\hat{\sigma}^2/b$. To see why, note that the first sum has conditional variance $\hat{\sigma}^2/b$, since the summands are independent and each has variance $\hat{\sigma}^2$. Similarly, the second sum has conditional variance $\hat{\sigma}^2/(b+\mu)\leq \hat{\sigma}^2/b$. Applying the Bernstein-type inequality for martingales discussed in the proof of \thmref{t:da-regret-highprob}, we get that with probability at least $1-\delta$,
\[
\sum_{j=1}^{\bar{m}}\frac{1}{b+\mu}\sum_{z\in Z_j}\bigl(f(w_j,z)-f(w^\star,z)\bigr)\leq
\sum_{j=1}^{\bar{m}}\frac{1}{b}\sum_{z\in \bar z_j}\bigl(f(w_j,z)-f(w^\star,z)\bigr) + O\left(\hat{\sigma}\sqrt{\frac{\bar{m}\log(1/\delta)}{b}}\right),
\]
where the $O$-notation hides only a (linear) dependence on the absolute bound over $|f(w,z)|$ for all $w,z$, that we assume to hold.

Combining this and \eqref{eqn:dist1-highprob} with a union bound, we get that with probability at least $1-\delta$,
\[
\sum_{j=1}^{\bar{m}}\sum_{z\in Z_j}\bigl(f(w_j,z)-f(w^\star,z)\bigr)
\leq (b+\mu)\psi\left(\frac{\sigma^2}{b}, \frac{\hat{\sigma}^2}{b},\delta,\frac{m}{b+\mu}\right)+
O\left((b+\mu)\hat{\sigma}\sqrt{\frac{\bar{m}\log(1/\delta)}{b}}\right).
\]
If $b+\mu$ divides $m$, then $\bar{m}=m/(b+\mu)$, and we get a bound of the form
\[
(b+\mu) \psi\left(\frac{\sigma^2}{b},\frac{\hat{\sigma}^2}{b},\delta,\frac{m}{b+\mu}\right)+
O\left(\hat{\sigma}\sqrt{\left(1+\frac{\mu}{b}\right)\log(1/\delta)m}\right).
\]
Otherwise, we repeat the ideas of \thmref{thm:minibatch} to get the regret bound.
\end{proof}

\bibliography{dmb_jmlr}

\end{document}